\documentclass[12pt]{article}
\usepackage[margin=.88in]{geometry}
\usepackage[utf8]{inputenc}
\usepackage{amsfonts}

\usepackage{amsmath,amssymb,amsfonts,amscd,amsxtra,amsthm}
\usepackage{graphicx}
 \usepackage{xcolor}
  
\usepackage[colorlinks,
linkcolor=blue,
anchorcolor=blue,
citecolor=blue]{hyperref}
  
\usepackage[round,authoryear]{natbib}

\newtheorem{thm}{Theorem}
\newtheorem{cor}{Corollary}
\newtheorem{lem}{Lemma}

\newtheorem{defn}{Definition}

\newtheorem{assum}{Assumption}

\newtheorem{remark}{Remark}

\def\eps{\varepsilon}
\def\what{\widehat}
\def\wt{\widetilde}

\newcommand{\MM}{\mathbb{M}}

\newcommand{\RR}{\mathbb{R}}
\newcommand{\PP}{\mathbb{P}}

\newcommand{\TT}{\mathbb{T}}
\newcommand{\II}{\mathbb{I}}

\newcommand{\EE}{\mathbb{E}}
\newcommand{\OO}{\mathbb{O}}

\newcommand{\SVDr}{\operatorname{SVD}_r}
\newcommand{\supp}{\operatorname{supp}}

\newcommand{\calA}{\mathcal{A}}
\newcommand{\calE}{\mathcal{E}}
\newcommand{\calN}{\mathcal{N}}
\newcommand{\calQ}{\mathcal{Q}} 
\newcommand{\calS}{\mathcal{S}}
\newcommand{\calP}{\mathcal{P}}
\newcommand{\calX}{\mathcal{X}}

\newcommand{\mcA}{\mathcal{A}}

\newcommand{\mcN}{\mathcal{N}}
\newcommand{\mcQ}{\mathcal{Q}} 
\newcommand{\mcS}{\mathcal{S}}
\newcommand{\mcP}{\mathcal{P}}

\newcommand{\whU}{\widehat{U}}
\newcommand{\whV}{\widehat{V}}

\newcommand{\whL}{\widehat{L}}
\newcommand{\whSig}{\widehat{\Sigma}}

\newcommand{\wtU}{\widetilde{U}}
\newcommand{\wtV}{\widetilde{V}}

\newcommand{\indep}{\perp \!\!\! \perp}
\newcommand{\ot}{\otimes}

\usepackage{algorithm} 
\usepackage{algpseudocode}

\title{Near-Optimal differentially private low-rank trace regression with guaranteed private initialization} 

\author{Mengyue, Zha  \footnote{Department of Mathematics, Hong Kong University of Science and Technology, mzha@connect.ust.hk. Mengyue, Zha's research was supported by Hong Kong PhD Fellowship Scheme. }} 

\date{(\today)}

\begin{document}
	\maketitle
\begin{abstract}
We study differentially private (DP) estimation of a rank-$r$ matrix $M \in \RR^{d_1\times d_2}$ under the trace regression model with Gaussian measurement matrices. Theoretically, the sensitivity of non-private spectral initialization is precisely characterized, and the differential-privacy-constrained minimax lower bound for estimating $M$ under the Schatten-$q$ norm is established. Methodologically, the paper introduces a computationally efficient algorithm for DP-initialization with a sample size of $n \geq \wt O (r^2 (d_1\vee d_2))$. Under certain regularity conditions, the DP-initialization falls within a local ball surrounding $M$. We also propose a differentially private algorithm for estimating $M$ based on Riemannian optimization (DP-RGrad), which achieves a near-optimal convergence rate with the DP-initialization and sample size of $n \geq \wt O(r (d_1 + d_2))$. Finally, the paper discusses the non-trivial gap between the minimax lower bound and the upper bound of low-rank matrix estimation under the trace regression model. It is shown that the estimator given by DP-RGrad attains the optimal convergence rate in a weaker notion of differential privacy. Our powerful technique for analyzing the sensitivity of initialization requires no eigengap condition between $r$ non-zero singular values. 
\end{abstract}

\section{Introduction}
The trace regression model \citep{rohde2011estimation, koltchinskii2011nuclear}, as an extension of the standard regression model, has been widely applied in various fields such as matrix completion, compressed sensing, and multi-task learning \citep{negahban2011estimation, koltchinskii2011nuclear, hamidi2022low}. Previous studies have proposed both convex and non-convex approaches for optimal estimation procedures for the model. However, the increasing demand for privacy protection has added new complexities to this extensively studied problem. Differential privacy (DP) \citep{dwork2006proposeDP}, a framework for protecting individual privacy, has been widely adopted in industrial and governmental applications \citep{google_privacy, apple2017, abowd2020modernization}. This paper aims to develop a near-optimal differentially private method for low-rank matrix estimation under the trace regression model. 
\paragraph{Trace regression model} Let $M\in \RR^{d_1\times d_2}$ be an unknown rank-$r$ matrix and $X_i\in \RR^{d_1\times d_2}$ be the measurement matrix for $i = 1, \cdots, n$. Suppose the noisy observation  $y_i$ satisfies 
\begin{equation}\label{equ: trace_regression_model}
	y_i = \left<X_i,M\right> + \xi_i \quad {\rm for} \quad i=1, \cdots, n, 
\end{equation}
where the model noise $\xi_i \overset{i.i.d.}{\sim} \calN(0, \sigma_{\xi}^2)$ and the inner product between $X_i$ and $M$ in Euclidean space is given by $\left<X_i, M\right> := \operatorname{Tr}(X_i^{\top}M)$. The goal of the present paper is to estimate the unknown rank-$r$ matrix $M\in \RR^{d_1\times d_2}$ under trace regression model defined by ($\ref{equ: trace_regression_model}$), subject to differential privacy, based on $n$ independent observations $Z:= \{(X_i, y_i)\}_{i=1}^n$. 

Our approaches are built upon the Gaussian mechanism \cite{dwork2006proposeDP}. The main difficulty in applying the Gaussian mechanism is sharply characterizing sensitivities of statistics whose privacy is under protection. Listed below are definitions of sensitivity, differential privacy (DP), and Gaussian mechanism. Interested readers may refer to \cite{dwork2006proposeDP, dwork2014algorithmic, vadhan2017complexity} for proofs and other details. Let $\|\cdot\|$ denotes the spectral norm and $\|\cdot\|_F$ denotes the Frobenius norm. 
\paragraph{Sensitivity} The sensitivity of a function $f$ that maps a dataset $Z$ into $\RR^{d_1\times d_2}$ is defined by $		\Delta_f:=\sup_{\textrm{neighbouring}(Z, Z')} \|f(Z)-f(Z')\|_{\rm F}, $ where and the supremum is taken over all neighbouring datasets $Z$ and $Z'$ that differ by at most one observation. 

\paragraph{Differential privacy} Let $\varepsilon > 0$ and $0< \delta< 1$, then we say the randomized algorithm $A$ is $(\varepsilon, \delta)$-differentially private if $ \PP\big(A(Z)\in \calQ\big)\leq e^{\eps}\PP\big(A(Z')\in\calQ\big)+\delta $ for all neighbouring data sets $Z, Z'$ and all subset $\calQ\subset \RR^{d_1\times d_2}$. 
\paragraph{Gaussian mechanism} \label{def: gaussian_mechanism} The randomized algorithm defined by $A(Z)=f(Z)+E$ is $(\varepsilon, \delta)$-DP where $E\in\RR^{d_1\times d_2}$ has i.i.d. $ N\big(0, 2\Delta_f^2\varepsilon^{-2}\log(1.25/\delta)\big)$ entries. 

\paragraph{RIP of Gaussian measurement matrices} The sensitivity of any statistic involving $\{X_i\}_{i\in[n]}$ depends on the properties of the measurement matrices. Besides, it has been previously established since \cite{candes2005decoding} that the restricted isometry property (RIP) on measurement matrices is crucial to the recovery of the unknown matrix $M$. Hence, assumptions on $\{X_i\}_{i\in[n]}$ are necessary and the present paper considers $\{X_i\}_{i\in[n]}$ with Gaussian design.
\begin{assum}[Gaussian design] \label{assumption_for_X}
	The vectorization of measurement matrices $X_1, \ldots, X_n$ are independent Gaussian $\operatorname{vec}(X_i) \sim \calN\left(\mathbf{0}, \Lambda_i\right)$ where $\Lambda_i$ 's are known, symmetric and positive definite. There exist absolute constants $C_l, C_u>0$ such that $	C_l \leq \lambda_{\min }\left(\Lambda_i\right) \leq \lambda_{\max }\left(\Lambda_i\right) \leq C_u .$	
\end{assum} 
The following Lemma \ref{lem: RIP_condition} shows that under Assumption \ref{assumption_for_X}, the measurement matrices $\{X_i\}_{i=1}^n$ satisfy the restricted isometry property (RIP) with high probability, see the proof in \ref{appx: lem_rip}. 
\begin{lem}\label{lem: RIP_condition}
Under the Assumption \ref{assumption_for_X}, for any $B\in \RR^{d_1 \times d_2}$ of rank $r$, there exist constants $c_1, c_2, c_3>0$ and $c_5>c_4>0$ such that if $n \geq c_1 r(d_1 + d_2)$, with probability at least $1-c_2 \exp \left(-c_3 r(d_1+d_2)\right)$, we have $c_4 \sqrt{C_u C_l}\|B\|_{\mathrm{F}}^2 \leq \frac{1}{n} \sum_{i=1}^n\left\langle X_i, B\right\rangle^2 \leq c_5 \sqrt{C_u C_l}\|B\|_{\mathrm{F}}^2.$
\end{lem}
\paragraph{Notations} Suppose $M$ is of rank-$r$ and its singular value decomposition is of the form $M = U \Sigma V^{\top}\in \RR^{d_1\times d_2}$ where $U\in \OO_{d_1, r}$, $V\in \OO_{d_2, r}$ and $\Sigma = \operatorname{diag}\{\sigma_1 \cdots \sigma_r\}$ with $\sigma_1 \geq  \cdots \geq \sigma_r$. Here, $\OO_{d, r}$ denotes the set of $d\times r$ matrices satisfying $H^{\top}H = I_r$. Let $\kappa := \sigma_1/\sigma_r$ be the condition number and $\kappa_{\xi}:= \sigma_{\xi}/\sigma_{r}$ be the signal-to-noise ratio.  Let $\wt O$ stand for the typical big-O notation up to logarithmic factors and $\wt O_p(\cdot)$ stand for $\wt O$ holds with high probability. 

\subsection{Main results} 

The paper presents several key results related to differentially private low-rank matrix estimation. Firstly, we propose a private initialization $\wt M_0$ (as detailed in Algorithm \ref{alg:DP-init}). Secondly, we establish the privacy-constrained minimax lower bound under the general Shatten-$q$ norm (as detailed in Theorem \ref{thm: fano_minimax_lower_bound}). Finally, we introduce a private estimator  $\wt M_{l^*}$ (as detailed in Algorithm \ref{alg:DP-RGrad}) that achieves the near-optimal convergence rate under the Frobenius norm. The sensitivity analysis of $\wt M_0$ heavily relies on a spectral representation formula for asymmetric matrices (See Lemma \ref{lem:spectral-formula}). 

We prove in Corollary \ref{cor: good_init} that the private initialization $\wt M_0$ satisfies $\| \wt M_0 - M \|_F \leq \sqrt{2r} \| \wt M_0 - M \| \leq c_0 \sigma_r, $ with high probability (w.h.p.), for a small constant $0<c_0<1$, provided that $ n \geq \wt O \left( (\kappa^4 r^2 + \kappa^2 \kappa_{\xi}^2r) (d_1\vee d_2) \right).$  
    
Theorem \ref{thm: fano_minimax_lower_bound} establishes the DP-constrained minimax risk of estimating the rank-$r$ matrix $M\in \RR^{d_1\times d_2}$ under model (\ref{equ: trace_regression_model}) and general Schatten-$q$ norm. Specifically, the minimax risk under Frobenius norm is in the order of $\sigma_{\xi} \sqrt{\frac{r(d_1\vee d_2)}{n}}+ \sigma_{\xi}\, \frac{r (d_1 \vee d_2)}{n\eps}$.  

Finally, we show in Theorem \ref{thm: convergence_rate} that with a sample size of $n\geq \wt O \left( \left(\kappa_{\xi}^2 \vee \kappa_{\xi} \right) r(d_1\vee d_2)  \right)$ and any initialization satisfying (\ref{equ: good_init}), Algorithm \ref{alg:DP-RGrad} achieves geometric convergence rate. The private estimator $\wt M_{l^*}$ attains the near-optimal convergence rate $$\left\|\wt M_{l^*} - M\right\|_F \leq \wt O_p\left(\sigma_{\xi} \sqrt{\frac{ r (d_1+d_2) }{n}} + (\sigma_{\xi} + \sigma_r) \frac{ r (d_1+d_2) }{n\eps} \right). $$ 

\subsection{Motivations and related works}
The trace regression model has been extensively researched, resulting in well-known optimal procedures and theoretical properties. Both convex \citep{rohde2011estimation, koltchinskii2011nuclear, candes2011tight, negahban2011estimation} and non-convex methods \citep{burer2003nonlinear, chen2015fast, zheng2016convergent, wei2016guarantees} have achieved the optimal convergence rate of the order $\sigma_{\xi} \sqrt{\frac{r\left(d_1 \vee d_2\right)}{n}}$ without the constraint from differential privacy. However, the DP-constrained minimax rate of low-rank matrix estimation under the trace regression model is still unknown. (Near) Optimal DP-algorithms have been developed for statistical problems such as learning Gaussians \cite{kamath2019privately, prettyfast2023, fastmean2023} or heavy-tailed distributions \cite{jonathan2020}, (sparse or generalized) linear regression \cite{Wang2018, cai2021cost, cai2023score}, and PCA \cite{blum2005dp-pca, dwork2014analyze, kamalika2012dp-pca, liu2022dp}. Previous works on DP-regression  \cite{cai2021cost, cai2023score} assume that all measurements have bounded $\ell_2$ norm. This assumption presents a significant limitation to studying the role of measurements play in the estimation error. Additionally, by treating measurements as a fixed vector or matrix, the statistical properties of measurements are disregarded. As a result, the opportunity for optimal statistical analysis subject to privacy concerns is inevitably lost. Recently, \citep{mcsherry2009differentially, liu2015fast, jain2018differentially, chien2021private, wang2023dp_mc} propose gradient-descent-based algorithms for DP low-rank matrix completion. These algorithms have attained near-optimal sample complexity. However, the problem of sample-efficient, differentially private initialization remains under-explored. Additionally, it is unknown how to establish the minimax risk of low-rank matrix estimation with the constraints of differential privacy, especially when the matrix is asymmetric. 

\subsection{Organization}
Section  \ref{sec: DP-initialization} proposes a DP-initialization algorithm and presents its privacy and utility guarantees. In Section \ref{sec: minimax lower bounds}, we establish a DP-constrained minimax lower bound (\ref{fano_lower_frobenius_norm}) for estimating the rank-$r$ matrix $M$ under the trace regression model. Section \ref{sec: the_upper_bound} presents the DP-estimator based on non-convex optimization and derives the upper bound of the DP-estimator's error, as stated in (\ref{upper_bound_M_l*}). We discuss the score attack argument and the non-trivial gap between the upper bound of (\ref{upper_bound_M_l*}) and the DP-constrained minimax lower bound (\ref{fano_lower_frobenius_norm}) in Section \ref{sec: weak_DP}. Proofs are given in Appendix \ref{appx: thm_dp_init} to \ref{appx: weak_dp}. 

\section{DP-initialization} \label{sec: DP-initialization}
Section \ref{algo-dp-init} presents an $(\eps, \delta)$-DP initialization $\wt M_0$, as stated in Algorithm \ref{alg:DP-init}. In Section \ref{sec: spectral_formula}, we introduce a spectral representation formula (See Lemma \ref{lem:spectral-formula}) that is crucial to sensitivity analysis on the initialization. With the help of the spectral representation formula, the privacy and utility guarantees of the DP-initialization $\wt M_0$ are given in Section \ref{sec: privacy_and_utility_guarantees}. 

\subsection{Algorithm for DP-initialization} \label{algo-dp-init}

We begin with $ \what L = n^{-1} \sum_{i=1}^n \operatorname{mat}\left(\Lambda_i^{-1}\operatorname{vec}\left(X_i\right)\right) y_i, $ which is unbiased for $M$. Suppose that the leading-$r$ left and right singular vectors of $\whL$ is given by the columns of $\whU\in \OO_{d_1, r}$ and $\whV\in \OO_{d_2, r}$, respectively. Then, $\what M_0 := \SVDr(\what L)$ is a non-private estimator for $M$. Let $\whSig := \whU^{\top} \whL \whV$, then we have $$\what M_0 = \SVDr(\what L) = \whU\whU^{\top} \whL \whV\whV^{\top} =  \whU\whSig\whV^{\top}. $$ 

It is reasonable to think about privatizing $\what U$, $\what V$, and $\what \Sigma$, separately. We first privatize the empirical spectral projector $\what U \what U^{\top}$ and $\what V \what V^{\top}$ by Gaussian mechanism. Thanks to post-processing property \cite{dwork2006proposeDP}, we obtain  $\wt U\in \OO^{d_1, r}$ and $\wtV\in \OO^{d_2, r}$ whose columns are differentially private and orthogonal. Secondly, we privatize the $r\times r$ matrix $\wtU^{\top} \whL \wtV$ by Gaussian mechanism and obtain $\wt \Sigma \in \RR^{r\times r}$ which is a private surrogate for $\what \Sigma = \what U^{\top} \what L \what V$. Finally, we take $\wt M_0 = \wt U\wt \Sigma \wt V^{\top}$ as the DP-initialization. We display the pseudo-code of the proposed DP-initialization in Algorithm \ref{alg:DP-init}. The privacy of $\wt M_0$ is guaranteed by the composition property \cite{dwork2006proposeDP}. 

\begin{algorithm}
	\caption{Differentially private initialization for trace regression}\label{alg:DP-init}
	\begin{algorithmic}
		\State{\textbf{Input}: the data set $\{(X_i, y_i)\}_{i=1}^n$; the covariance matrices $\{\Lambda_i\}_{i=1}^{n}$; sensitivity $\Delta^{(1)}$,  $\Delta^{(2)}>0$; rank $r$; nuisance variance $\sigma_{\xi}^2$;  privacy budget $\varepsilon>0$ and $0< \delta < 1$.}
		\State{\textbf{Output}: $(\varepsilon,\delta)$-differentially private initialization $\wt{M}_0$. }
		\State{Compute the unbiased sample estimator $\what L$ and its top-$r$ left and right singular vectors:
			$$
			\what L \leftarrow n^{-1} \sum_{i=1}^n \operatorname{mat}\left(\Lambda_i^{-1}\operatorname{vec}\left(X_i\right)\right) y_i \,\quad {\rm and}\,\quad \what M_0 = \what U \what \Sigma \what V \leftarrow {\rm SVD}_r(\what L)
			$$
			Compute $(\varepsilon/3,  \delta/3)$-differentially private singular subspaces by adding artificial Gaussian noise:
			$$
			\widetilde U\leftarrow {\rm SVD}_r\Big(\what U\what U^{\top}+E_U\Big)\, {\rm with}\, (E_U)_{ij}=(E_U)_{ji} \stackrel{{\rm i.i.d.}}{\sim} N\Big(0,  \frac{18\Delta^{(1)^2}}{\varepsilon^2}\log(3/\delta)\Big), \, \forall 1\leq i\leq j\leq d_1
			$$
			$$
			\widetilde V\leftarrow {\rm SVD}_r\Big(\what V\what V^{\top}+E_V\Big)\, {\rm with}\, (E_V)_{ij}=(E_U)_{ji} \stackrel{{\rm i.i.d.}}{\sim} N\Big(0,  \frac{18\Delta^{(1)^2}}{\varepsilon^2}\log(3/\delta)\Big), \, \forall 1\leq i\leq j\leq d_2
			$$
			Compute $(\varepsilon/3,  \delta/3)$-differentially private estimates of singular values up to rotations:
			\begin{align*}
				\wt \Sigma &\, \leftarrow \wt U^{\top} \what L \wt V + E_{\Sigma}\, {\rm with}\, (E_{\Sigma})_{ij}=(E_{\Sigma})_{ji}\stackrel{i.i.d.}{\sim} N\Big(0,  \frac{18\Delta^{(2)^2}}{\varepsilon^2}\log(3/\delta)\Big),\, \forall 1\leq i\leq j\leq r
			\end{align*} 
			Compute $(\varepsilon,  \delta)$-differentially private initialization: $ \wt M_0 \leftarrow \wt U\wt \Sigma \wt V^{\top} $ \\
			\textbf{Return}: $\wt M_0$}
	\end{algorithmic}
\end{algorithm}

To this end, we define the sensitivities of $\Delta^{(1)}$ and $\Delta^{(2)}$ appear in Algorithm \ref{alg:DP-init}. Let 
$$\what L^{(i)} := n^{-1} \sum_{j\neq i}^n \operatorname{mat}\left(\Lambda_j^{-1}\operatorname{vec}\left(X_j\right)\right) Y_j + n^{-1}\operatorname{mat}\left(\Lambda_i^{-1}\operatorname{vec}\left(X_i^{\prime}\right)\right) y_i^{\prime}, $$ 
where $(X_i^{\prime}, y_i^{\prime})$ is an i.i.d. copy of $(X_i, y_i)$. Then, the estimator $\what L^{(i)}$ differs with $\whL$ only by the $i$-th pair of observations. Suppose the top-$r$ left and right singular vectors of $\what L^{(i)}$ are given by $U^{(i)}$ and $V^{(i)\top}$, respectively. The sensitivity of $\what U \what U^{\top}$ is defined by
\begin{align*}
    \Delta_{\what U \what U^{\top}} & =  \sup_{{\rm neighbouring} (Z, Z^{\prime})} \left\| \what U (Z) \what U(Z)^{\top} - \what U(Z^{\prime}) \what U(Z^{\prime})^{\top} \right\|_F = \max_{i\in[n]} \left\|  \what U \what U^{\top} - \what U^{(i)} \what U^{(i)\top}  \right\|_F, 
\end{align*}
and the sensitivity $ \Delta_{\what V \what V^{\top}}$ of $\what V \what V^{\top}$ is defined similarly. We refer to $\Delta^{(1)} \triangleq \Delta_{\what U \what U^{\top}}\vee  \Delta_{\what V \what V^{\top}} $ as the sensitivity of singular subspaces and define the sensitivity $$\Delta^{(2)} \triangleq \Delta_{\wtU^{\top} \whL \wtV} = \sup_{{\rm neighbouring} (Z, Z^{\prime})} \left\| \wt U  \what L(Z)  \wt V^{\top} - \wt U \what L(Z^{\prime})  \wt V^{\top} \right\|_F = \max_{i\in [n]} \left\| \wt U^{\top} \left( \what L - \what L^{(i)} \right) \wt V \right\|_F. $$ 
 
As privatizing $\what \Sigma = \what U^{\top} \what L \what V$ by Gaussian mechanism, the scale of artificial noise avoids growing with an unnecessary $\sqrt{d_1}\vee \sqrt{d_2}$ but rather growing with a smaller quantity $\sqrt{r}$. This benefit motivates us to privatize $\what U$, $\what V$ and $\what \Sigma$, separately. However, it is technically challenging to characterize $\Delta^{(1)} = \Delta_{\what U \what U^{\top}}\vee  \Delta_{\what V \what V^{\top}}$ due to the non-linear dependence of $\what U \what U^{\top}$ and $\what V \what V^{\top}$ on the dataset $Z = \{(X_i, y_i)\}_{i=1}^n$. To address this challenge, we introduce an explicit spectral representation formula (See Lemma \ref{lem:spectral-formula}) to obtain a sharp upper bound on the sensitivity of the singular subspaces. 

\subsection{Spetral representation formula}\label{sec: spectral_formula}
This section introduces a spectral representation formula for asymmetric matrices (See Lemma \ref{lem:spectral-formula}). To begin with, we quickly explain the standard \textit{symmetric dilation} trick (See e.g., Section 2.1.17 in \cite{tropp2015introduction}) and define auxiliary operators used in Lemma \ref{lem:spectral-formula}.
\paragraph{Symmetric dilation and auxiliary operators} For any $M\in \RR^{d_1\times d_2}$, the symmetric dialation $M_*$ of $M$ is a $(d_1+d_2)\times (d_1+d_2)$ matrix defined by 
$M_* := \left(\begin{array}{cc}
	0 & M \\
	M^{\top} & 0
\end{array}\right)$. It is easy to check that $M_* = M_*^{\top}$ and $\|M_*\| = \|M\|$. Further, if we assume that $M$ is of rank $r$ and has the form of SVD $M = U\Sigma V^{\top} \in \RR^{d_1\times d_2}$, then $M_*$ is of rank-$2r$ and has eigendecomposition of the form 
$$
\frac{1}{\sqrt{2}}\left(\begin{array}{cc}
	U & U \\
	V & -V
\end{array}\right)\cdot\left(\begin{array}{cc}
	\Sigma & 0 \\
	0 & -\Sigma
\end{array}\right)\cdot \frac{1}{\sqrt{2}}\left(\begin{array}{cc}
	U & U \\
	V & -V
\end{array}\right)^{\top}:= U_{M^*} \Sigma_{M^*} U_{M^*}^{\top}. 
$$
The $2r$ eigenvectors of $M_*$ is given by the columns of $U_{M^*}\in \OO_{(d_1+d_2), 2r}$. For integer $t\geq 1$, we define operators $$Q^{-t} := U_{M^*}\Sigma_{M^*}^{-t} U_{M^*}^{\top} \quad {\rm and} \quad Q^{-0} :=Q^{\perp}\triangleq(U_{M^*})_{\perp}(U_{M^*})_{\perp}^{\top}=I_{d_1+d_2}-U_{M^*}U_{M^*}^{\top}. $$

\begin{lem}[Spectral representation formula]
	\label{lem:spectral-formula} 
	Let $M\in \RR^{d_1\times d_2}$ be any rank-$r$ matrix with singular values $\sigma_1 \geq \cdots \geq \sigma_r >0 $ and $\what L=M + \Delta \in \RR^{d_1\times d_2}$ be a perturbation of $M$ where $\Delta \in \RR^{d_1\times d_2}$ is the deviation matrix. Suppose the top-$r$ left and right singular vectors of $\what L$ and $M$, are given by the columns of $\what U$, $\what V$ and $U$, $V$,  respectively. Suppose that $2\|\Delta\|\leq \sigma_r$, then 
	\begin{equation*}
		\left(\begin{array}{cc}
			\what U \what U^{\top} - UU^{\top} & 0 \\
			0 & \what V \what V^{\top} - VV^{\top}
		\end{array}\right)=\sum_{k \geq 1} \mathcal{S}_{M_*, k}(\Delta_*).
	\end{equation*}
	Here, $\Delta_*$ is the symmetric dilation of $\Delta := \what L - M$ and the $k$-th order term $\mathcal{S}_{M_*, k}(\Delta_*)$ is a summation of $\binom{2k}{k}$ terms defined by $		\mathcal{S}_{M_*, k}(\Delta_*)=\sum_{\mathbf{s}: s_1+\ldots+s_{k+1}=k}(-1)^{1+\tau(\mathbf{s})} \cdot Q^{-s_1} \Delta_* Q^{-s_2} \ldots \Delta_* Q^{-s_{k+1}}, $
	where $\mathbf{s}=\left(s_1, \ldots, s_{k+1}\right)$ contains non-negative indices and $\tau(\mathbf{s})=\sum_{j=1}^{k+1} \mathbb{I}\left(s_j>0\right).$ 
\end{lem}

In Lemma \ref{lem:spectral-formula}, the spectral projectors $\what U \what U^{\top}$ and $\what V \what V^{\top}$ of the matrix $\what L = M+ \Delta$, is explicitly represented in terms of the symmetric dilation of $\Delta$, with the help of auxiliary operators $Q^{-0}$ and $Q^{-t}$ for integer $t\geq 1$. The proof of Lemma \ref{lem:spectral-formula} is deferred to Appendix \ref{proof_lem_spectral-formula}. Note that Lemma \ref{lem:spectral-formula} accommodates a diverging condition number and requires no eigengap condition between $r$ non-zero singular values. In the proof of Theorem \ref{thm: dp_init}, we shall see that $\what V \what V^{\top}-\what V^{(i)} \what V^{(i)\top}$ and $\what U \what U^{\top}-\what U^{(i)} \what U^{(i)\top}$ are mainly contributed by the 1-st order approximation $\calS_{M_*, 1}(\Delta_*) - \calS_{M_*, 1}(\Delta_*^{(i)}) $ where $\Delta_*^{(i)}$ is the symmetric dilation of $\Delta^{(i)} := \what L^{(i)} - M$.   

\subsection{Privacy and utility guarantees of the initialization} \label{sec: privacy_and_utility_guarantees}
In this section, we study the privacy and utility guarantees of the initialization $\wt M_0$. Theorem \ref{thm: dp_init} characterizes the sensitivities $\Delta^{(1)}$  and $\Delta^{(2)}$ needed to guarantee an $(\varepsilon, \delta)$-DP $\wt M_0$, and  present the upper bounds of $\| \wt M_0 - M \|$ and $\| \wt M_0 - M \|_F$. The proof of Theorem \ref{thm: dp_init} is provided in Appendix \ref{appx: thm_dp_init}. 
\begin{thm} [Privacy and utility guarantees of the initialization $\wt M_0$] \label{thm: dp_init}
	Consider i.i.d. observations $Z = \{z_1, \cdots, z_n\}$ drawn from the trace regression model stated in $(\ref{equ: trace_regression_model})$ where $z_i := (X_i, y_i)$ for $i=1, \cdots, n$. Let the true rank-$r$ regression coefficients matrix be $M\in\RR^{d_1\times d_2}$. Suppose that $\{X_i\}_{i\in[n]}$ satisfy the Assumption \ref{assumption_for_X}. Under the mild condition $n \geq \frac{\log^2 n}{(d_1 \vee d_2)\log (d_1+d_2)}$, there exists absolute constants $C_1, C_2, C_3>0$ such that $$n\geq n_0 \triangleq C_1  C_l^{-1} r \left( \frac{\sigma_{\xi} + \sqrt{C_u r} \sigma_1}{\sigma_r} \right)^2  (d_1\vee d_2) \log (d_1 + d_2 ) ; $$ if Algorithm \ref{alg:DP-init} takes in sensitivities at least $\Delta^{(1)} = C_2  \left( \frac{\sqrt{C_l^{-1}} \left(  \sqrt{C_u r}  \sigma_1 + \sigma_{\xi} \right) }{\sigma_r}\right)\frac{\sqrt{r}}{n}\log n $ and
    $\Delta^{(2)} = C_2 \sqrt{C_l^{-1}} \left(  \sqrt{C_u r}  \sigma_1 + \sigma_{\xi} \right) \frac{\sqrt{r}}{n} \log n, $ then Algorithm \ref{alg:DP-init} is guaranteed to be $(\eps, \delta)$-DP. Moreover, the output $\wt M_0$ of Algorithm \ref{alg:DP-init} satisfies 
	\begin{align*}
		& \| \wt M_0 - M \| \bigvee  \left(\| \wt M_0 - M \|_F/ \sqrt{2r} \right) \\
		& \leq \underbrace{C_3 \sqrt{C_l^{-1}} \left(  \sqrt{C_u r}  \sigma_1 + \sigma_{\xi} \right) \, \frac{\sigma_1}{\sigma_r} \sqrt{\frac{(d_1 \vee d_2)\log (d_1 + d_2)}{n}}}_{e_1} \\
		& \quad + \underbrace{C_3  \sqrt{C_l^{-1}} \left(  \sqrt{C_u r}  \sigma_1 + \sigma_{\xi} \right) \left( \frac{\sigma_1}{\sigma_r}  \frac{\sqrt{r (d_1 \vee d_2)}}{n\varepsilon} + \frac{r}{n\eps} \right)  \log n \log^{\frac{1}{2}}(\frac{3.75}{\delta})}_{e_2} , 
	\end{align*}
	with probability at least $1- (d_1+d_2)^{-10} - n^{-9} - \exp(-d_1) - \exp(-d_2) - 10^{-20r}$. 
\end{thm}

In Theorem \ref{thm: dp_init}, the sample size condition  $n\geq n_0$ ensures that the spectral norm of perturbations is small enough, i.e., $\|\Delta\|+ \max_{i\in[n]} \|\Delta^{(i)}\| \leq \sigma_r /2, $ to apply Lemma \ref{lem:spectral-formula} and obtain a sharp characterization on  $\Delta^{(1)} \triangleq \Delta_{\what U \what U^{\top}}\vee  \Delta_{\what V \what V^{\top}}$. Theorem \ref{thm: dp_init} also provides an upper bound on the sensitivity of \textit{pseudo singular values}, which is of the order $\Delta^{(2)} \triangleq \Delta_{\wtU^{\top} \whL \wtV} \asymp \sigma_1 \Delta^{(1)}$. Based on these results, Algorithm \ref{alg:DP-init} outputs an $(\eps, \delta)$-DP initialization $\wt M_0$ under the sample size condition $$n\geq \wt O \left((\kappa^2 r^2 + \kappa_{\xi}  r) (d_1\vee d_2)\right) , $$ with an upper bound on the error $\|\wt M_0 - M\|$ consisting of two terms. The first term $e_1$ accounts for the statistical error of $\what M_0$ and is greater than the the optimal rate $\sigma_{\xi}\sqrt{\frac{d_1 \vee d_2}{n}} $ \citep{koltchinskii2011neumann}. The second term $e_2$ can be further decomposed into the cost of privacy on the singular subspaces which is of the order $ \wt O_p\left(\frac{\sigma_1}{\sigma_r} \left(\sigma_1 \sqrt{r} + \sigma_{\xi}  \right) \frac{\sqrt{(d_1\vee d_2)}}{n\varepsilon} \right)$, and the cost of privacy arises from privatizing the singular values by Gaussian mechanism which is of the order $\wt O_p\left( (\sigma_1  \sqrt{r}+\sigma_{\xi})r / (n\eps) \right)$. 

Next, Corollary \ref{cor: good_init} gives the sample size required by a DP-initialization $\wt M_0$ that falls within a local ball of $M$. The proof of Corollary \ref{cor: good_init} is deferred to Appendix \ref{appx: cor_good_init}. 

\begin{cor}\label{cor: good_init}
	Under the conditions stated in Theorem \ref{thm: dp_init}, as the sample size is sufficiently large 
	\begin{align*}
		n\geq C_1 \max \Bigg\{ & \underbrace{ C_l^{-1} \left(\frac{  \sqrt{C_u r}  \sigma_1 + \sigma_{\xi} }{\sigma_r}\right)^2 \kappa^2 r (d_1\vee d_2) \log(d_1 + d_2)}_{n_1}, \\
		& \underbrace{ \sqrt{C_l^{-1}}  \left(\frac{   \sqrt{C_u r}  \sigma_1 + \sigma_{\xi} }{\sigma_r}\right) \left( \kappa r\sqrt{d_1\vee d_2} + r^{\frac{3}{2}}\right) \log n \frac{\log^{\frac{1}{2}}(\frac{3.75}{\delta})}{\varepsilon}}_{n_2}\Bigg \}, 
	\end{align*} 
	for some absolute constant $c_2>0$, then we have, for some small constant $0<c_0<1$, 
	\begin{equation} \label{equ: good_init}
		\| \wt M_0 - M \|_F \leq \sqrt{2r} \| \wt M_0 - M \| \leq c_0 \sigma_r. 
	\end{equation}
\end{cor}

In Corollary \ref{cor: good_init}, the error due to $\|M\|\cdot \left( \left\|  \what V \what V^{\top} - V V^{\top}\right\| + \left\| \what U \what U^{\top} - UU^{\top} \right\| \right) $ dominates the statistical error $\|\what L -M\|$ and the sample size $n_1$ is required to control these two terms; the sample size $n_2$ controls the error due to privatizing the singular subspaces and singular values. According to Corollary \ref{cor: good_init}, as the sample size $ n \geq \wt O \left( (\kappa^4 r^2 + \kappa^2 \kappa_{\xi}^2r) (d_1\vee d_2) \right), $ the $(\eps, \delta)$-DP $\wt M_0$ is guaranteed to fall into a local ball surrounding $M$, as stated in (\ref{equ: good_init}). The condition (\ref{equ: good_init}) is a pre-requisite for Algorithm \ref{alg:DP-RGrad} to converge geometrically, as discussed in Theorem \ref{thm: convergence_rate}. 

\section{Minimax lower bounds} \label{sec: minimax lower bounds} 
This section applies DP-Fano's lemma (See Lemma \ref{lem:dp-fano}) to establish the DP-constrained minimax lower bound of estimating the matrix $M\in \MM_r:= \{M\in \RR^{d_1\times d_2}: \operatorname{rank}(M) = r\}$ under trace regression model  
\begin{equation} \label{dist: prob_trace_model}
	f_M(y_i | X_i) = \frac{1}{\sqrt{2 \pi} \sigma_{\xi}} \exp \left(\frac{-\left(y_i-\left<X_i, M\right> \right)^2}{2 \sigma^2}\right); X_i\sim \calN\left(\mathbf{0}, \Lambda_i \right). 
\end{equation}
Suppose we observe an i.i.d. sample $\{(X_i, y_i), (X_i^{\prime}, y_i^{\prime})\}_{i\in[n]}$ of size $2n$ drawn from (\ref{dist: prob_trace_model}). Then, we have $$\Bar{y}_i: = y_i + y_i^{\prime} = \left< \left(\begin{array}{cc}
	0 & M \\
	M^{\top} & 0
\end{array}\right) , \left(\begin{array}{cc}
	0 & X_i \\
	X_i^{\prime\top} & 0
\end{array}\right) \right> + \xi_i + \xi_i^{\prime},$$ where the underlying matrix $M_*$. Let $f(X_i, X_i^{\prime})$ be the joint distribution of $X_i$ and $X_i^{\prime}$; $f_{M_*}(\Bar{y}_i \mid X_i, X_i^{\prime})$ be the conditional distribution of $\Bar{y}_i$ given $X_i, X_i^{\prime}$; and denote the joint distribution of $\Bar{y}_i$ and $X_i, X_i^{\prime}$ as $f_{M_*}(\Bar{y}_i, X_i, X_i^{\prime})$. It is clear that $f_{M_*}(\Bar{y}_i \mid X_i, X_i^{\prime})$ is given by the distribution of $$\mcN(\left<\left(\begin{array}{cc}
	0 & M \\
	M^{\top} & 0
\end{array}\right), \left(\begin{array}{cc}
	0 & X_i \\
	X_i^{\prime\top} & 0
\end{array}\right)  \right>, 2\sigma_{\xi}^2).$$ Let $\ot$ represent the tensor product of marginal laws. For a given matrix $\Sigma = \operatorname{diag}\{\sigma_1, \cdots, \sigma_r\}$ where $C\sigma \geq \sigma_1 \cdots \geq \sigma_r \geq c\sigma$ for some constants $\sigma >0$ and $C>c>0$, we consider the family of normal distribution under trace regression model:  
$$
\mcP_{\Sigma}:=\Big\{\bigotimes_{i=1}^n f_{M_*}(\Bar{y}_i, X_i, X_i^{\prime}): M_*=(U\Sigma V^{\top})_*, U\in \OO_{d_1, r}, V\in \OO_{d_2, r}  \Big\}. 
$$
By definition,  each distribution $P_{M_*}\in\mcP_{\Sigma}$ is indexed by $U_{M_*} = \frac{1}{\sqrt{2}}\left(\begin{array}{cc}
	U & U \\
	V & -V
\end{array}\right) \in \OO_{d_1+d_2, 2r}$ whose columns are the $r$ eigenvectors of $M_*$. Next, we employ DP-Fano's lemma to derive the minimax lower bound of estimating $M_*$ by a sample drawn from $\mcP_{\Sigma}$. Let ${\rm KL}(\cdot \| \cdot)$ and ${\rm TV}(\cdot, \cdot)$ denote the Kullback-Leibler divergence and total variation distance between two distributions.

\begin{lem}[DP-Fano's lemma, \cite{acharya2021differentially}]\label{lem:dp-fano} 
	Let $\mcP := \{P: P = \mu^{(1)} \times \cdots \times \mu^{(n)} \}$ be a family of product measures indexed by a parameter from a pseudo-metric space $(\Theta, \rho)$. Denote $\theta(P)\in\Theta$ the parameter associated with the distribution $P$.  Let $\mcQ=\{P_1,\cdots, P_N\}\subset \mcP$ contain $N$ probability measures  and there exist constants $\rho_0, l_0, t_0>0$ such that  for all $i\neq i^{\prime}\in[N]$, $
	\rho\left(\theta(P_i), \theta(P_{i^{\prime}})\right) \geqslant \rho_0, \quad \operatorname{KL}\left(P_{i} \| P_{i^{\prime}}\right) \leq l_0, \quad
	\sum_{k\in[n]} \operatorname{TV}\left(\mu_{i}^{(k)} , \mu_{i^{\prime}}^{(k)}\right) \leq t_0, $
	where $P_i = \mu_i^{(1)} \times \cdots \times \mu_i^{(n)} $ and $P_{i^{\prime}} = \mu_{i^{\prime}}^{(1)} \times \cdots \times \mu_{i^{\prime}}^{(n)} $. Suppose $\delta \lesssim e^{-n}$, then 
	\begin{align}\label{eq:fano-bd1}
		\inf_{A \in \mcA_{\varepsilon,\delta}(\mcP)} & \sup_{P\in\mcP } \EE_{A} \; \rho( A , \theta(P)) \geqslant \max \left\{\frac{\rho_0}{2}\left(1-\frac{ l_0 +\log 2}{\log N}\right), \frac{\rho_0}{4}\left(1 \bigwedge \frac{N-1}{\exp \left( 4 \varepsilon t_0 \right)}\right)\right\},
	\end{align}
	where the infimum is taken over all the $(\varepsilon,\delta)$-DP randomized algorithm defined by $\mcA_{\varepsilon,\delta}(\mcP) := \{ A:Z\mapsto \Theta \ {\rm and}\   A  \text{ is } \; (\varepsilon, \delta)\text{-differentially private} \; \text{for all} \; Z\sim P \in\mcP \ \} $ . 
\end{lem} 

To apply Fano's lemma,  we need to construct a large subset with well-separated elements for $\OO_{d_1+d_2, 2r}$. By Lemma \ref{local_packing_set}, there exists a subset $\mcS_q^{(d_1+d_2)}\subset \OO_{d_1+d_2,2r}$ with cardinality $\left|\mcS_q^{(d_1+d_2)}\right|\geq 2^{2r(d_1+d_2-2r)}$ such that for any $H\neq H^{\prime}\in \mcS_q^{(d_1+d_2)}$, $$\|HH^{\top}-H^{\prime}H^{\prime \top}\|_q \gtrsim \tau \eps_0 (2r)^{1/q} \quad {\rm and} \quad
\|HH^{\top}-H^{\prime}H^{\prime\top}\|_{\rm F}\lesssim 2 \sqrt{r}\varepsilon_0,$$ for some small constants $\tau, \varepsilon_0 >0$, where $\|\cdot\|_q$ denotes the Schatten-q norm. We then consider the family of distributions 
\begin{align*}
	\mcP_{\sigma} = \left\{\bigotimes_{i=1}^n f_{M_*}(\Bar{y}_i, X_i, X_i^{\prime}): M_* = \sigma HH^{\top}, H \in \mcS_q^{(d_1+d_2)}\right\} \subset \mcP_{\Sigma},
\end{align*}
whose cardinality $N:= \left| \mcP_{\sigma} \right| \geq 2^{2r(d_1+d_2-2r)}$. Let $M_* = \sigma HH^{\top}$ and $M_*^{\prime} = \sigma H^{\prime} H^{\prime \top}$. As shown in Appendix \ref{proof_of_thm_fano_minimax_lower_bound}, for any $H\neq H^{\prime} \in  \mcS_q^{(d_1+d_2)}$, we have $$	 \operatorname{KL} \left(\bigotimes_{i=1}^n f_{M_*}(\Bar{y}_i, X_i, X_i^{\prime}) \|\bigotimes_{i=1}^n f_{M_*^{\prime}}(\Bar{y}_i, X_i, X_i^{\prime}) \right) \lesssim \frac{n}{\sigma_{\xi}^2} C_u \sigma^2 r \eps_0^2, $$
and $\sum_{k\in[n]}\operatorname{\operatorname{TV}}\left(f_{M_*}(\Bar{y}_i, X_i, X_i^{\prime}),  f_{M^{\prime}_*}(\Bar{y}_i, X_i, X_i^{\prime}) \right) \lesssim n \frac{\sqrt{C_u} \sigma}{\sigma_{\xi}} \sqrt{r}\eps_0. $
To this end, we obtain Theorem \ref{thm: fano_minimax_lower_bound} by applying Lemma \ref{lem:dp-fano} with the bounded KL divergence and TV distance, together with the facts that $\mcP_{\sigma}\subset \mcP_{\Sigma}$. The proof of Theorem \ref{thm: fano_minimax_lower_bound} is deferred to Appendix \ref{proof_of_thm_fano_minimax_lower_bound}. 
\begin{thm} \label{thm: fano_minimax_lower_bound}
Consider a sample of size $n$ drawn from the distribution $P_{M^*} \in \mcP_{\Sigma}$, then for any $\delta\lesssim e^{-n}$ and any $q\in[1,\infty]$, there exists a constant $c>0$ $$\inf_{\wt M } \sup_{P \in \mcP_{\Sigma}} \EE\Big\|\wt M - M \Big\|_q \geq c\frac{\sigma_{\xi}}{\sqrt{C_u}}\left (r^{1/q}\sqrt{\frac{d_1 \vee d_2 }{n}}+r^{\frac{1}{2}+\frac{1}{q}}\frac{d_1 \vee d_2 }{n\varepsilon}\right)\bigwedge r^{1/q} \sigma, $$
where the infimum is taken over all possible $(\varepsilon,\delta)$-DP algorithms. It suffices to choose $q=1, 2, \infty$ to obtain the bounds in the nuclear norm, Frobenius norm, and spectral norm, respectively. For example, when $q=2$, there exists a constant $c>0$
\begin{equation} \label{fano_lower_frobenius_norm}
     \inf_{\wt M } \sup_{P \in \mcP_{\Sigma}} \EE\Big\|\wt M - M \Big\|_F \geq c\Big(\underbrace{\frac{ \sigma_{\xi}}{\sqrt{C_u}} \sqrt{\frac{r(d_1 \vee d_2) }{n}}}_{l_1}+ \underbrace{\frac{\sigma_{\xi}}{\sqrt{C_u}}\frac{r(d_1 \vee d_2) }{n\varepsilon}}_{l2} \Big)\bigwedge r^{1/2} \sigma. 
\end{equation}
\end{thm} 
In Theorem \ref{thm: fano_minimax_lower_bound}, the lower bound (\ref{fano_lower_frobenius_norm}) consists of two terms, the statistical error $l_1$ and the cost of privacy $l_2$. The next section proposes a DP-estimator that attains the minimax lower bound  (\ref{fano_lower_frobenius_norm}), up to an additional factor $\sigma_r$ and some logarithmic factors. As a supplement to DP-Fano's Lemma which works for $\delta \lesssim e^{-n}$, we also try the score attack argument, which is valid for a wider range of $\delta \lesssim n^{1+\gamma}$ where $\gamma>0$ is a constant. Theorem \ref{thm: minimax_lower_bound} presents the DP-constrained lower bound established by the score attack argument.  The content and proof of Theorem \ref{thm: minimax_lower_bound} are deferred to Appendix \ref{sec: score_the_lower_bound}.  We also point out that it is trivial to derive the minimax lower bound of the case $d_1=d_2=d$ based on DP-Fano's Lemma since there is no need to apply the trick of symmetrization. 

\section{Upper bounds with differential privacy} \label{sec: the_upper_bound} 

In this section, we present Algorithm \ref{alg:DP-RGrad}, DP-RGrad, and show that DP-RGrad attains the near-optimal convergence rate for differentially privately estimating low-rank matrices under the trace regression model. Our approach is based on privatizing the Riemannian gradient descent (RGrad) by the Gaussian mechanism. Interested readers may refer to \cite{vandereycken2013low, edelman1998geometry, adler2002, AbsMahSep2008} for the basics of RGrad. Let the estimate we obtain after $l$ iterations be the rank-$r$ matrix  $M_l\in \RR^{d_1\times d_2}$ whose SVD has the form $M_l=U_l \Sigma_l V_l^{\top}$. It is well-known in \cite{AbsMahSep2008, vandereycken2013low} that the tangent space of $\MM_r$ at $M_l$ is given by $\mathbb{T}_l:=\left\{Z \in \mathbb{R}^{d_1 \times d_2}: Z=U_l R^{\top}+L V_l^{\top}, R \in \mathbb{R}^{d_2 \times r}, L \in \mathbb{R}^{d_1 \times r}\right\}$. The projection of the gradient $G_l$ onto $\mathbb{T}_l$ is
$\mathcal{P}_{\mathbb{T}_l}\left(G_l\right)=U_l U_l^{\top} G_l+G_l V_l V_l^{\top}-U_l U_l^{\top} G_l V_l V_l^{\top},$ which is of rank at most $2 r$. Let the noisy gradient descent on the tangent space be $M_l-\eta_l \mathcal{P}_{\mathbb{T}_l}\left(G_l \right) + \mathcal{P}_{\mathbb{T}_l} N_l$ where $N_l \in \RR^{d_1 \times d_2}$ is the Gaussian noise matrix. Then, we retract it back to $\MM_r$ and obtain 
\begin{equation} \label{updates}
    M_{l+1}=\SVDr \left( M_l-\eta_l \mathcal{P}_{\mathbb{T}_l}\left(G_l\right) + \mathcal{P}_{\mathbb{T}_l} N_l \right).
\end{equation}
We update the estimate as defined in (\ref{updates}) for $1= 0, \cdots,  l^*-1$ where $l^*$ is the total number of iterations. Thanks to the composition property and Gaussian mechanism, we only need to ensure that each iteration is $(\eps/ l^*, \delta/ l^*)$-DP. For trace regression model defined in (\ref{equ: trace_regression_model}), empirical mean squared loss is defined as $ \mathcal{L}_n(M_l; Z) := \frac{1}{2n} \sum_{i=1}^n\left(\left\langle X_i, M_l \right\rangle-y_i\right)^2$ and the empirical Euclidean gradient is $G_l:= \nabla\mathcal{L}_n(M_l; Z)=\frac{1}{n} \sum_{i=1}^n\left(\left\langle X_i, M_l\right\rangle-y_i\right) X_i. $ The sensitivity of the $l$-th iteration is 
$\Delta_{l} :=  \max_{{\rm neighbouring} (Z, Z^{\prime})} \left\| M_l - \eta \calP_{\TT_l}(G_l(Z)) - \left[  M_l - \eta \calP_{\TT_l}(G_l(Z^{\prime})) \right] \right\|_F. $

\begin{algorithm}
	\caption{DP-RGrad for trace regression}\label{alg:DP-RGrad}
	\begin{algorithmic}
		\State{\textbf{Input}: the loss function $\mathcal{L}$;  the data set $\{(X_i, y_i)\}_{i=1}^n$; sensitivities $\{ \Delta_l \}_{l\in[l^*]}$; DP-initialization $\wt M_0$; rank $r$; nuisance variance $\sigma_{\xi}^2$;  privacy budget $\varepsilon>0,  \delta \in(0,1)$.}
		\State{\textbf{Output}: $(\varepsilon,\delta)$-differentially private estimate $M_{l^*}$ for trace regression. }
		\State{\textbf{Initialization}: $M_0 \longleftarrow \wt M_0. $} 
		\State{ \For{$l+1 \in [l^*]$}{
				$$
				M_{l+1} \longleftarrow \operatorname{SVD}_r\left(M_l-\eta_l  \mathcal{P}_{\mathbb{T}_l}\left(G_l\right)+ \mathcal{P}_{\mathbb{T}_l}N_l\right),
				$$
				where $G_l$ is the empirical Euclidean gradient 
				$$
				G_l:= \nabla\mathcal{L}_n(M_l; Z)=\frac{1}{n} \sum_{i=1}^n\left(\left\langle X_i, M_l\right\rangle-y_i\right) X_i,  
				$$
				and $N_{l}\in \RR^{d_1 \times d_2}$ has entries i.i.d. to $$\calN \left(0, \frac{2\Delta_l^2 l^{*2}}{\eps^2}\log \left(\frac{1.25 l^*}{\delta}\right) \right). $$} }
    \State{\textbf{Return}: $\wt M_{l^*}\longleftarrow M_{l^*}$}
    \end{algorithmic}
\end{algorithm}

Theorem \ref{thm: convergence_rate} establishes the error bound of the estimator $\wt M_{l^*}$ given by Algorithm \ref{alg:DP-RGrad}. The proof of Theorem \ref{thm: convergence_rate} is deferred to Appendix \ref{appx: proof_of_thm_convergence_rate}. 

\begin{thm} \label{thm: convergence_rate}
Consider i.i.d. observations $Z = \{z_1, \cdots, z_n\}$ drawn from the trace regression model stated in $(\ref{equ: trace_regression_model})$ where the true low-rank regression coefficients matrix being $M\in\MM_r$. Here, $z_i := (X_i, y_i)$ for $i=1, \cdots, n$ and we assume that $\{X_i\}_{i\in[n]}$ satisfy the Assumption \ref{assumption_for_X}. Under the Assumption \ref{assumption_for_X} and the condition that $(d_1+d_2)>\log n$, suppose that Algorithm \ref{alg:DP-RGrad} takes in an $(\varepsilon, \delta)$-DP initialization such that for some small constant $0<c_0<1$, $ \| \wt M_0 - M \|_F \leq \sqrt{2r} \| \wt M_0 - M \| \leq c_0 \sigma_r, $ and the sensitivities $\Delta_l$ take the value  
	\begin{align*}
		\Delta_l & = C_3 \frac{\eta}{n} (\sigma_{\xi}+ \sigma_r \sqrt{C_u}) \sqrt{C_u r (d_1+d_2)\log n}, \quad {\rm for\; all} \quad  l = 1, \cdots, l^*, 
	\end{align*}
	for some absolute constant $C_3>0$, then we have, Algorithm \ref{alg:DP-RGrad} is $(2\eps, 2\delta)$-differentially private. Moreover, as the sample size 
	\begin{align*}
		n \geq c_4 \max \Bigg\{& \underbrace{ c_1 r(d_1+d_2)}_{n_3}, \quad \underbrace{ \eta^2 \kappa_{\xi}^2 C_u r (d_1+d_2) \log (d_1+d_2) }_{n_4}, \\
		&\underbrace{ \eta \sqrt{C_u } \left(\kappa_{\xi}+  \sqrt{C_u}\right)  r (d_1+d_2) \log^{3/2} (n)  \frac{ \log^{1/2}\left( \frac{1.25 \log(n)}{\delta}\right) }{\eps}   }_{n_5}\Bigg\}, 
	\end{align*} 
	for some small constant $0<c_4<1$, number of iteration $l^* = O(\log n)$, and the step size $0<\eta<1$, we have the output $\wt M_{l^*}$ of Algorithm \ref{alg:DP-RGrad} satisfies  
	\begin{align*}
		\left\|\wt M_{l^*} - M\right\|_F & \leq \underbrace{C_4 \sigma_{\xi} \sqrt{ C_u}  \sqrt{\frac{ r (d_1+d_2) }{n}}  \log^{1/2}(d_1+d_2) }_{u_1}   \\
		& + \underbrace{ C_4 \sqrt{C_u }(\sigma_{\xi} + \sigma_r \sqrt{C_u}) \frac{ r (d_1+d_2) }{n\eps} \log^{3/2} n \log^{1/2}\left( \frac{1.25 \log (n)}{\delta} \right) }_{u_2} . 
	\end{align*} 
	with probability at least 
	\begin{align*}
		1  &  -\wt c_2 \exp \left(-\wt c_3 r(d_1+d_2)\right) -(d_1+d_2)^{-10}-n^{-9} - e^{-d_1} - e^{-d_2} - 10^{-20r} \\ 
		& - \left( (d_1+d_2)^{-10} + \exp(-(d_1+d_2)) + n^{-9} + \exp\left( - 10 C_u(d_1+d_2) \right)n^{-9} \right)\log n. 
	\end{align*} 
\end{thm}

According to Theorem \ref{thm: convergence_rate}, the upper bound of $\left\|\wt M_{l^*} - M\right\|_F$ can be decomposed into the the statistical error $u_1$ and the cost of privacy $u_2$. The term $u_1$ matches the the optimal rate $l_1 \sim \sigma_{\xi}\sqrt{\frac{r(d_1 \vee d_2)}{n}}$, only up to logarithmic factors. However, the term  $u_2$ differs from the theoretical lower bound of the cost of privacy $l_2 \sim \sigma_{\xi} \frac{ r (d_1+d_2) }{n\eps}$, by a non-trivial factor $\sigma_r$, apart from logarithmic factors. In conclusion, Theorem \ref{thm: convergence_rate} shows that as the sample size $ n\gtrsim \wt O \left( \left(\kappa_{\xi}^2 \vee \kappa_{\xi} \right) r(d_1\vee d_2)  \right), $ the estimator $\wt M_{l^*}$ given by Algorithm \ref{alg:DP-RGrad} attains the near-optimal convergence rate 
\begin{equation} \label{upper_bound_M_l*}
	\wt O_p\left( \sigma_{\xi} \sqrt{\frac{ r (d_1+d_2) }{n}}+ (\sigma_{\xi} + \sigma_r) \frac{ r (d_1+d_2) }{n\eps}   \right).  
\end{equation}
The sample size requirement of Theorem \ref{thm: convergence_rate} has the following explanations. The sample size $n_3$ is required to guarantee that the RIP condition stated in Lemma \ref{lem: RIP_condition} occurs with high probability. The sample size $n_4$ is necessary to control the statistical error contributed by $\sum_{i\in[n]} \xi_i X_i$ in each iteration where $\xi_i$ is the model noise. The sample size $n_5$ arises from controlling the cost of privacy due to $\calP_{\TT_l}N_l$ in each iteration.

\section{Discussion} \label{sec: weak_DP}
In this section, we discuss the non-trivial gap $\sigma_r$ between $	u_2 \sim (\sigma_{\xi} + \sigma_r) \frac{ r (d_1+d_2) }{n\eps}$ and $l_2 \sim \sigma_{\xi}\, \frac{r (d_1 \vee d_2)}{n\eps}$. Note that $l_2$ is free of $ \sigma_r $ while $u_2$ contains the factor  $\sigma_r$ arising from sensitivities 
\begin{align*}
	\Delta_l \asymp \frac{\eta}{n} (\sigma_{\xi}+ \sigma_r \sqrt{C_u}) \sqrt{C_u r (d_1+d_2)\log n} \quad {\rm for }\quad l = 1, \cdots, l^*. 
\end{align*}
The quantity $\sigma_r\sqrt{C_u}$ of $\Delta_l$ arises from $\left\|\left<M_l-M, X_i\right> \right\|_{\psi_2} \leq \sqrt{C_u} \left\| 
M_l - M \right\|_F $, as elaborated in (\ref{single_term_in_GD_sensitivity}). Here, $\|\cdot\|_{\psi_2}$ denotes the sub-Gaussian norm. Therefore, we cannot get rid of the factor $\sigma_r$ once the measurement matrices $\{X_i\}_{i\in[n]}$ are subject to differential privacy. In many real applications, however, the measurement matrices  $\{X_i\}_{i\in[n]}$ are fixed with deterministic designs.  People publish $\{X_i\}_{i\in[n]}$ to the public with little concern on the privacy of $\{X_i\}_{i\in[n]}$. Although the exposure of  $\{X_i\}_{i\in [n]}$ alone will not reveal any information on $M$, the privacy of $M$ suffers from leakage when the public has access to the joint observations $\{(X_i, y_i)\}_{i\in[n]}$. We, therefore, introduce the following notion of privacy for neighboring datasets sharing the same measurement matrix. 
\begin{defn}[weak $(\varepsilon, \delta)$-differential privacy] 
The algorithm $A$ that maps $Z$ into $\RR^{d_1 \times d_2}$ is \textit{weak} $(\varepsilon, \delta)$-differentially private over the dataset $Z$ if $	\PP\big(A(Z)\in \calQ\big)\leq e^{\eps}\PP\big(A(Z')\in\calQ\big)+\delta,$ for all neighbouring data set $Z, Z'$ sharing the same measurement $X$ and all subset $\calQ\subset \RR^{d_1\times d_2}$. 
\end{defn} 
In Theorem \ref{thm: weak_dp_convergence_rate}, Appendix \ref{appx: weak_dp}, we show that as $n\gtrsim \wt O \left( \left(\kappa_{\xi}^2 \vee \kappa_{\xi} \right) r(d_1\vee d_2)  \right), $ the estimator $\wt M_{l^*}$ given by Algorithm \ref{alg:DP-RGrad} attains the optimal convergence rate $	\wt O_p\left( \sigma_{\xi} \sqrt{\frac{ r (d_1+d_2) }{n}}+ \sigma_{\xi} \frac{ r (d_1+d_2) }{n\eps} \right) $ in the sense of weak differential privacy. The analogs of Theorem \ref{thm: dp_init}, Corollary \ref{cor: good_init} and Theorem \ref{thm: convergence_rate} in the sense of weak differential privacy can be found as Theorem \ref{thm: weak_dp_init}, Corollary \ref{cor: weak_good_init} and Theorem \ref{thm: weak_dp_convergence_rate} in Appendix \ref{appx: weak_dp}.  It is interesting to explore in future work whether the score attack argument or DP-Fano's Lemma can be generalized to include the non-trivial factor $\sigma_r$. 

\section*{Acknowledgement}
The author would like to express sincere gratitude to Hong Kong PhD Fellowship Scheme and Hong Kong RGC GRF Grant 16301622 for providing financial support for this research. The author also wishes to acknowledge the invaluable guidance provided by Prof. Dong, Xia throughout the research process. Additionally, the author would like to extend heartfelt thanks to Mr. Zetao, Fei for his constructive criticism during the paper revision. Their contributions have been instrumental in the successful completion of this research. 


\newpage

\setcounter{page}{1}
\setcounter{section}{0}
\setcounter{equation}{0}
\renewcommand{\theequation}{S.\arabic{equation}}
\renewcommand{\thesection}{\Alph{section}}

\appendix

\section{Proof of Theorem \ref{thm: dp_init}}\label{appx: thm_dp_init}
The proof of Theorem \ref{thm: dp_init} consists of four parts. In Part \ref{thm_1_part_1}, we list several existing results that are useful in the proofs later. In Part \ref{thm_1_part_2}, Lemma \ref{lem:spectral-formula} works as the main technique to derive the sensitivity $\Delta^{(1)}$. Part \ref{thm_1_part_3} derives the sensitivity $\Delta^{(2)} $. Part \ref{thm_1_part_4} establishes the upper bounds of $\|\wt M_0 - M\|$ and $\| \wt M_0 -M \|_F$ based on the $\Delta^{(1)}$ and $\Delta^{(2)}$. 

\subsection{Part 1} \label{thm_1_part_1}

The following Theorem \ref{thm: bernstein} proposed by Proposition $2$, \cite{koltchinskii2011neumann} will be frequently used to establish the upper bound of the spectral norm of a summation of independent random matrices. 

\begin{thm}[Bernstein's inequality, \cite{koltchinskii2011neumann}] \label{thm: bernstein}
	Let $B_1, \cdots, B_n$ be independent $d_1 \times d_2$ matrices such that for some $\alpha \geq 1$ and all $i\in[n]$
	$$
	\mathbb{E} B_i=0, \quad\left\|\Lambda_{\max }\left(B_i\right)\right\|_{\Psi_\alpha}=: K<+\infty .
	$$
	Let 
	$$
	S^2:=\max \left\{\Lambda_{\max }\left(\sum_{i=1}^n \mathbb{E} B_i B_i^{\top}\right) / n, \Lambda_{\max }\left(\sum_{i=1}^n \mathbb{E} B_i^{\top} B_i\right) / n\right\} .
	$$
	Then, for some constant $C>0$ and for all $t>0$,
	$$
	\mathbb{P}\left(\left\|\frac{1}{n}\sum_{i=1}^n B_i\right\| \geq C S \sqrt{\frac{t+\log (d_1+d_2)}{n}}+C K \log ^{1 / \alpha}\left(\frac{K}{S}\right) \frac{t+\log (d_1+d_2)}{n}\right) \leq \exp (-t). 
	$$
\end{thm}

Theorem \ref{thm: bernstein} applies to bound the spectral norm of  $\Delta: = \what L - M$. The existing result for the case of heavy-tailed noise can be found in Theorem 6, \cite{shen2023computationally}. Adapting the existing result to the case of Gaussian noise, we have that for some absolute constant $C_0 >0$, 
\begin{equation}\label{Delta_bounded_by_bernstein}
	\|\Delta\| = \|\what L -M\| \leq C_0 \sqrt{C_l^{-1}} \left( \sigma_{\xi} + \sqrt{C_u} \sqrt{r} \sigma_1 \right) \sqrt{\frac{(d_1 \vee d_2)\log (d_1 + d_2)}{n}} , 
\end{equation} 
with probability at least $1-(d_1+d_2)^{-10}$. The following Lemma originated from Lemma 18, \cite{shen2023computationally}, is useful to analyze the matrix permutation due to singular value decomposition. 

\begin{lem}[Matrix Permutation, \cite{shen2023computationally}]\label{lem: matrix_permutation}
	Let $M \in \mathbb{R}^{d_1 \times d_2}$ be a rank-$r$ matrix with singular value decomposition of the form $M=U \Sigma V^{\top}$ where $\Sigma=\operatorname{diag}\left\{\sigma_1, \sigma_2, \cdots, \sigma_r\right\}$ with $\sigma_1 \geq \sigma_2 \geq \cdots \geq \sigma_r>0$. For any $\what M \in \mathbb{R}^{d \times d}$ satisfying $\|\what M-M\|_{\mathrm{F}}<\sigma_r / 8$, then
	\begin{align*}
		\left\|\operatorname{SVD}_r(\what M)-M\right\| & \leq\|\what M-M\|+40 \frac{\|\what M-M\|^2}{\sigma_r}, \\
		\left\|\operatorname{SVD}_r(\what M)-M\right\|_{\mathrm{F}} & \leq\|\what M-M\|_{\mathrm{F}}+40 \frac{\|\what M-M\|\|\what M-M\|_{\mathrm{F}}}{\sigma_r},
	\end{align*}
	and 
	$$\left\|\what U \what U^{\top}-UU^{\top}\right\| \leq \frac{8}{\sigma_r}\|\what M-M\|, \quad\left\|\what V \what V^{\top}-V V^{\top}\right\| \leq \frac{8}{\sigma_r}\|\what M-M\|, $$
	where the leading $r$ left singular vectors of $\what M$ are given by the columns of  $\what U \in \mathbb{R}^{d_1 \times r}$ and the leading $r$ right singular vectors of $\what M$ are given by the columns of  $\what V \in \mathbb{R}^{d_2 \times r}$. 
\end{lem}

\subsection{Part 2}\label{thm_1_part_2}
The second part aims to derive the sensivitity $$\Delta^{(1)}:= \max_{i\in[n]} \left( \| \what U\what U^{\top} - \what U^{(i)} \what U^{(i)\top}\|_F \vee \|\what V\what V^{\top} - \what V^{(i)} \what V^{(i)\top}\|_F \right). $$ 
Before moving on, we present Lemma \ref{lem:con-bounds}, which provides conclusions on $\Delta$ and $\Delta^{(i)}$, frequently used in the proof later. The proof of Lemma \ref{lem:con-bounds} can be found in Appendix \ref{proof_lem_cond_bounds}. 

\begin{lem}\label{lem:con-bounds}
	Under model (\ref{equ: trace_regression_model}), Assumption \ref{assumption_for_X}, and the condition $n \geq \frac{\log^2 n}{(d_1 \vee d_2)\log (d_1+d_2)}$, there exists some absolute constant $C_0, C_1>0$ such that the event 
	\begin{align*}
		\calE_*:= & \Bigg\{\max_{i\in[n]}\big\|\Delta - \Delta^{(i)}\big\| \leq C_0\cdot  n^{-1}\sqrt{ C_l^{-1}}\left( \sqrt{C_u r}\sigma_1 + \sigma_{\xi}\right) \log n \Bigg\} \\
		& \bigcap \Bigg\{  \|\Delta\|+\max_{i\in[n]}\|\Delta^{(i)}\| \leq C_0 \sqrt{C_l^{-1}} \left( \sigma_{\xi} + \sqrt{C_u} \sqrt{r} \sigma_1 \right) \sqrt{\frac{(d_1 \vee d_2)\log (d_1 + d_2)}{n}} \Bigg\}, 
	\end{align*}
	holds with probability at least $1-(d_1+d_2)^{-10} - n^{-9}$. Conditioned on the event $\calE_*$, as the sample size satisfies 
	\begin{equation} \label{n_to_control_Delta_1}
		n\geq C_1 C_l^{-1} r \left( \frac{\sigma_{\xi} + \sqrt{C_u r} \sigma_1}{\sigma_r} \right)^2  (d_1\vee d_2) \log (d_1 + d_2 ), 
	\end{equation}
	we have 
	\begin{equation} \label{trigger}
		\|\Delta_*\|\vee \max_{i\in[n]} \|\Delta_*^{(i)}\| = 
		\|\Delta\|\vee \max_{i\in[n]}  \|\Delta^{(i)}\|\leq \frac{\sigma_r}{8\sqrt{2r}} < \frac{\sigma_r}{5+\delta} <  \frac{\sigma_r}{2},
	\end{equation}
	and $$ \|\Delta_*\|_F \vee \max_{i\in[n]} \|\Delta_*^{(i)}\|_F = \max_{i\in[n]} 
	\|\Delta\|_F \vee\|\Delta^{(i)}\|_F  \leq \frac{\sigma_r}{8}, $$ for some constant $\delta>0$, where $\Delta_*^{(i)}$ is the symmetric dilation of $\Delta^{(i)} := L^{(i)} - M$. 
\end{lem} 

The following analysis is proceeded under the sample size condition (\ref{n_to_control_Delta_1}) and is mainly conditioned on the event $\calE_*$ which happens with probability at least $1-(d_1+d_2)^{-10}-n^{-9}$. \\

\noindent\emph{Step 1: expansion. } Conditioned on $\calE_*$, we are able to apply Lemma \ref{lem:spectral-formula} to $\Delta_*$ and $\Delta_*^{(i)}$ and get
\begin{equation*}
	\left(\begin{array}{cc}
		\what U \what U^{\top} - \what U^{(i)} \what U^{(i)\top} & 0 \\
		0 & \what V \what V^{\top} - \what V^{(i)} \what V^{(i)\top}
	\end{array}\right)=\sum_{k \geq 1} \mathcal{S}_{M_*, k}(\Delta_*) - \sum_{k \geq 1} \mathcal{S}_{M_*, k}(\Delta^{(i)}_*). 
\end{equation*}
Our goal is to bound $\Delta^{(1)}$ which satisfies
\begin{align*}
	\Delta^{(1)} & \leq \max_{i\in[n]} \left( \|  \what U \what U^{\top} - \what U^{(i)} \what U^{(i)\top} \|_F + \| \what V \what V^{\top} - \what V^{(i)} \what V^{(i)\top} \|_F \right) \\
	& \leq \max_{i\in[n]} \left( \left\|  \mathcal{S}_{M_*, 1}(\Delta_*) - \mathcal{S}_{M_*, 1}(\Delta^{(i)}_*) \right\|_F + \left\| \sum_{k \geq 2} \mathcal{S}_{M_*, k}(\Delta_*) - \sum_{k \geq 2} \mathcal{S}_{M_*, k}(\Delta^{(i)}_*) \right\|_F\right). 
\end{align*} 
\noindent\emph{Step 2: bounding the first order term.} By the definition of $\mathcal{S}_{M_*, 1}(\Delta_*)$ and $ \mathcal{S}_{M_*, 1}(\Delta^{(i)}_*)$, 
\begin{equation} \label{equ: first_order_approximation}
	\begin{aligned}
		&  \max_{i\in[n]} \|   \mathcal{S}_{M_*, 1}(\Delta_*) - \mathcal{S}_{M_*, 1}(\Delta^{(i)}_*)\|  =  \max_{i\in[n]} \|  Q^{-1}  (\Delta-\Delta^{(i)})^{\top} Q_{\perp} +  Q_{\perp} (\Delta-\Delta^{(i)})^{\top} Q^{-1}\| \\
		& \leq \frac{2 \sqrt{r}}{\sigma_r} \max_{i\in[n]} \|  \Delta-\Delta^{(i)}\| \leq  C_4 \frac{\sqrt{r}}{n\sigma_r}\sqrt{ C_l^{-1}}\left( \sqrt{C_u r}\sigma_1 + \sigma_{\xi}\right) \log n, 
	\end{aligned} 
\end{equation}
conditioned on the event $\calE_*$, for some absolute constant $C_4>0$. 

\noindent\emph{Step 3: bounding the higher order terms.} Let $I_k$ be the index set for terms in $\mathcal{S}_{M_*, k}$
$$I_k=\left\{\mathbf{s}: \mathbf{s}=\left(s_1, \ldots, s_{k+1}\right), \sum_{m=1}^{k+1} s_m=k, s_m \geq 0 \quad \forall m \in[k+1]\right\}, $$
with the cardinality $\left|I_k\right|=\binom{2k}{k}$. We define
$$
\mathcal{T}_{M_*, k, \mathbf{s}, l}\left(\Delta_*-\Delta_*^{(i)}\right):=Q^{-s_1} \Delta_*^{(i)} Q^{-s_2} \cdots Q^{-s_l}\left(\Delta_*-\Delta_*^{(i)}\right) Q^{s_{l+1}} \cdots Q^{-s_k} \Delta_* Q^{s_{k+1}},
$$
for $k \geq 2, \mathbf{s}=\left(s_1, \cdots, s_{k+1}\right) \in I_k$ and $l \in[k]$. Since $\left|I_k\right|=\binom{2k}{k}$, the higher order terms $\max_{i\in[n]}\left\|\sum_{k \geq 2} \mathcal{S}_{M_*, k}(\Delta_*)-\sum_{k \geq 2} \mathcal{S}_{M_*, k}\left(\Delta_*^{(i)}\right)\right\|_{\mathrm{F}}$ 
\begin{equation}\label{decomp_higher_order_term_1}
	\begin{aligned}
		=\max_{i\in[n]}\left\|\sum_{k \geq 2} \sum_{\mathbf{s} \in I_k} \sum_{l \in[k]} \mathcal{T}_{M_*, k, \mathbf{s}, l}\left(\Delta_*-\Delta_*^{(i)}\right)\right\|_{\mathrm{F}}  \leq\max_{i\in[n]} \sum_{k \geq 2}\binom{2k}{k} \sum_{l \in[k]}\left\|\mathcal{T}_{M_*, k, \mathbf{s}, l}\left(\Delta_* -\Delta_*^{(i)}\right)\right\|_{\mathrm{F}}.  
	\end{aligned}
\end{equation}
It is sufficient to find a upper bound of $\left\|\mathcal{T}_{M_*, k, \mathbf{s}, l}\left(\Delta_*-\Delta_*^{(i)}\right)\right\|_{\mathrm{F}}$. Denote
$$
D_{\max }:= C_1 \sqrt{C_l^{-1}} \left( \sigma_{\xi} + \sqrt{C_u} \sqrt{r} \sigma_1 \right) \sqrt{\frac{(d_1 \vee d_2)\log (d_1 + d_2)}{n}},
$$
which appeared in the event $\calE_*$ as an upper bound of $\|\Delta\|+\max _{i \in[n]}\left\|\Delta^{(i)}\right\|$.

Conditioned on $\calE_*$, for all $i\in[n]$, $k\geq 2, \mathbf{s}\in I_k$ and $l\in[k]$, 
\begin{align*}
	\left\|\mathcal{T}_{M_*, k, \mathbf{s}, l}\left(\Delta-\Delta^{(i)}\right)\right\|_{\mathrm{F}}\leq \sqrt{2r} \left\|\mathcal{T}_{M_*, k, \mathbf{s}, l}\left(\Delta-\Delta^{(i)}\right)\right\| \leq \frac{\sqrt{2r}}{\sigma_r}\left\| \Delta_* - \Delta_*^{(i)} \right\|\left( \frac{D_{\max}}{\sigma_r} \right)^{k-1}, 
\end{align*}
where the first inequality is because $\mathcal{T}_{M_*, k, \mathbf{s}, l}\left(\Delta_*-\Delta_*^{(i)}\right)$ is of rank at most $2r$. Let $a$ be a function defined by $a(k) = \binom{2k}{k}k$, then $a(2)=12$ and $\frac{a(k+1)}{a(k)} \leq 5$ for all integer $k\geq 2$,  
\begin{equation} \label{decomp_higher_order_term_2}
	\begin{aligned}
		& \max_{i\in[n]}\sum_{k \geq 2}\binom{2k}{k}\sum_{l \in[k]}\left\|\mathcal{T}_{M_*, k, \mathbf{s}, l}\left(\Delta_* -\Delta_*^{(i)}\right)\right\|_{\mathrm{F}}\\
		& \leq \max_{i\in[n]}\binom{4}{2} \sum_{k\geq 0} 5^k \left( \frac{\|D_{\max}\|}{\sigma_r} \right)^k  \frac{\sqrt{2r}}{\sigma_r}\left\|\Delta_* - \Delta_*^{(i)}\right\| \left( \frac{D_{\max}}{\sigma_r} \right)\\
		& \leq \max_{i\in[n]} a(2)\frac{\sqrt{2r}}{\sigma_r}\left\| \Delta_* - \Delta_*^{(i)} \right\| \sum_{k\geq 0}  \left(\frac{5}{5+\delta}\right)^k \left( \frac{D_{\max}}{\sigma_r} \right)\\
		& \leq \max_{i\in[n]} a(2) \left( \frac{5+\delta}{\delta} \right) \left( \frac{D_{\max}}{\sigma_r} \right) \frac{\sqrt{2r}}{\sigma_r}\left\|\Delta_* - \Delta_*^{(i)}\right\|, 
	\end{aligned}
\end{equation}
where the last step is due to (\ref{trigger}), which is guaranteed by the sample size condition (\ref{n_to_control_Delta_1}) together with the event $\calE_*$. Combining (\ref{decomp_higher_order_term_1}) and (\ref{decomp_higher_order_term_2}), since  $\left\|\Delta_* - \Delta_*^{(i)}\right\| = \left\|\Delta - \Delta^{(i)}\right\|$, conditioned on the event $\calE_*$, we have 
\begin{align*}
	& \max_{i\in[n]} \left\|\sum_{k \geq 2} \mathcal{S}_{M_*, k}(\Delta_*)-\sum_{k \geq 2} \mathcal{S}_{M_*, k}\left(\Delta_*^{(i)}\right)\right\|_{\mathrm{F}} \leq C_4 \left( \frac{12}{\delta} \right) \frac{\sqrt{2r}}{n} \sqrt{ C_l^{-1}}\left( \frac{\sqrt{C_u r} \, \sigma_1 + \sigma_{\xi} }{\sigma_r}\right) \log n, 
\end{align*}
for some absolute constant $C_3>0$. In conclusion, conditioned on $\calE_*$, as the sample size
$ n\geq C_1 C_l^{-1} r \left( \frac{\sigma_{\xi} + \sqrt{C_u r} \sigma_1}{\sigma_r} \right)^2  (d_1\vee d_2) \log (d_1 + d_2), $ for some absolute constant $C_1>0$, we have $	\Delta^{(1)}  \leq C_4 \frac{\sqrt{r}}{n} \sqrt{ C_l^{-1}}\left( \frac{\sqrt{C_u r} \, \sigma_1 + \sigma_{\xi} }{\sigma_r}\right) \log n, $ for some absolute constant $C_4$. 

Let $E_U\in \RR^{d_1\times d_1}$ be a symmetric matrix where the entries $(E_U)_{ij}$ i.i.d. to $\calN(0, \frac{18\Delta^{(1)^2}}{\eps^2} \log(\frac{3.75}{\delta}))$ for $1\leq i\leq j \leq d_1$. Then, conditioned on the event $\calE_*$ and (\ref{n_to_control_Delta_1}), for some absolute constant $\wt C_4>0$, $	\|E_U\|\leq \wt C_4 \frac{\sqrt{r d_1}}{n\varepsilon} \sqrt{ C_l^{-1}}\left( \frac{\sqrt{C_u r} \, \sigma_1 + \sigma_{\xi} }{\sigma_r}\right)\log n \log^{\frac{1}{2}}(\frac{3.75}{\delta}), $ with probability at least $1-e^{-d_1}$. Moreover, by Gaussian mechanism, $\what U \what U^{\top} + E_U$ is $(\varepsilon/3, \delta/3)$-DP and thus $\wt U$ is also  $(\varepsilon/3, \delta/3)$-DP thanks to the post-processing property of differential privacy. Furthermore, by Davis-Kahan's Theorem, for some absolute constant $\tilde{c_0}>0$  
\begin{equation} \label{U_dp_permutation}
	\begin{aligned}
		\left\|\wt U \wt U^{\top}-UU^{\top}\right\| & \leq 1 \wedge \left( \left\| \what U \what U^{\top} - UU^{\top} \right\| + \|E_U\| \right) \stackrel{(a)}{\leq} 1 \wedge \left( \left( \frac{8}{\sigma_r}\|\what L -M\| \wedge 1  \right) + \|E_U\| \right) \\
		& \leq 1 \wedge \tilde{c_0} \Bigg( \sqrt{C_l^{-1}} \left( \frac{\sqrt{C_u r} \, \sigma_1 + \sigma_{\xi} }{\sigma_r}\right)  \sqrt{\frac{(d_1 \vee d_2)\log (d_1 + d_2)}{n}}  \\
		& \quad \quad \quad \quad +\sqrt{ C_l^{-1}}\left( \frac{\sqrt{C_u r} \, \sigma_1 + \sigma_{\xi} }{\sigma_r}\right) \frac{\sqrt{r d_1}}{n\varepsilon}  \log n \log^{\frac{1}{2}}(\frac{3.75}{\delta}) \Bigg), 
	\end{aligned}
\end{equation}
where we apply Lemma \ref{lem: matrix_permutation} in step $(a)$. 

Let $E_V\in \RR^{d_2\times d_2}$ be a symmetric matrix with $(E_V)_{ij}$ i.i.d. to $\calN(0, \frac{18\Delta^{(1)^2}}{\eps^2} \log(\frac{3.75}{\delta}))$ for $1\leq i\leq j \leq d_2$, then for some absolute constant $\wt C_4 >0$,  conditioned on the event $\calE_*$ and (\ref{n_to_control_Delta_1}), we have $	\|E_V\|\leq \wt C_4 \frac{\sqrt{r d_2}}{n\varepsilon} \sqrt{ C_l^{-1}}\left( \frac{\sqrt{C_u r} \, \sigma_1 + \sigma_{\xi} }{\sigma_r}\right) \log n \log^{\frac{1}{2}}(\frac{3.75}{\delta}), $
with probability at least $1-e^{-d_2}$. Moreover, by Gaussian mechanism, $\what V \what V^{\top} + E_U$ is $(\varepsilon/3, \delta/3)$-DP and $\wt V$ is also  $(\varepsilon/3, \delta/3)$-DP thanks to the post-processing property of differential privacy. Furthermore, by Davis-Kahan's Theorem, for some absolute constant $\tilde{c_0}>0$ 
\begin{equation} \label{V_dp_permutation}
	\begin{aligned}
		\left\|\wt V \wt V^{\top}-VV^{\top}\right\| & \leq 1 \wedge \tilde{c_0} \Bigg( \sqrt{C_l^{-1}}\left( \frac{\sqrt{C_u r} \, \sigma_1 + \sigma_{\xi} }{\sigma_r}\right)  \sqrt{\frac{(d_1 \vee d_2)\log (d_1 + d_2)}{n}}  \\
		& \quad \quad \quad \quad +\sqrt{ C_l^{-1}}\left( \frac{\sqrt{C_u r} \, \sigma_1 + \sigma_{\xi} }{\sigma_r}\right) \frac{\sqrt{r d_2}}{n\varepsilon}  \log n \log^{\frac{1}{2}}(\frac{3.75}{\delta}) \Bigg). 
	\end{aligned}
\end{equation}

\subsection{Part 3}\label{thm_1_part_3}
Given the $(\varepsilon/3, \delta/3)$-DP singular vectors $\wt U\in \OO^{d_1\times r}$ and $\wt V\in \OO^{d_2\times r}$ obtained in Part \ref{thm_1_part_2}, we derive the sensitivity $ \Delta^{(2)} : = \max_{i\in [n]} \left\| \wt U^{\top} \left( \what L - \what L^{(i)} \right) \wt V \right\|_F. $ Conditioned on the event $\calE_*$, 
\begin{align*}
	\Delta^{(2)} & : = \max_{i\in [n]} \left\| \wt U^{\top} \left( \what L - \what L^{(i)} \right) \wt V \right\|_F \leq  \max_{i\in [n]} \sqrt{r}  \left\| \what L - \what L^{(i)}  \right\| \\
	& = \max_{i\in [n]} \sqrt{r}  \left\| \what L - M + M - \what L^{(i)}  \right\| = \max_{i\in [n]} \sqrt{r}  \left\| \Delta - \Delta^{(i)}  \right\| \\
	& \leq  C_3\cdot  n^{-1}\sqrt{ C_l^{-1}} \sqrt{r}\left( \sqrt{C_u r}\sigma_1 + \sigma_{\xi}\right) \log n. 
\end{align*} 
Let $E_{\Sigma}$ be a $r\times r$ matrix with entries i.i.d. to $\calN(0, 18\Delta^{(2)^2} \log (\frac{3.75}{\delta})/\eps^2 )$, then  
\begin{equation} \label{equ: operator_norm_E_Sigma}
	\| E_{\Sigma}\| \leq \wt C_4 \cdot  \frac{r}{n\eps} \sqrt{ C_l^{-1}}\left( \sqrt{C_u r}\sigma_1 + \sigma_{\xi}\right) \log n \log^{\frac{1}{2}}\left(\frac{3.75}{\delta}\right),
\end{equation}
for some absolute constant $\wt C_4>0$ with probability at least $10^{-20r}$. Moreover, by Gaussian mechanism, $\wt \Sigma = \wt U^{\top} \what L \wt V + E_{\Sigma}$ is $(\varepsilon/3, \delta/3)$-differentially private. Thanks to the composition property of differential privacy, the output of Algorithm \ref{alg:DP-init} 
$\wt M_0 = \wt U  \wt \Sigma  \wt V^{\top}, $ is $(\eps, \delta)$-differentially private. 

\subsection{Part 4}\label{thm_1_part_4} 
In this part, we derive the upper bound of $\| \wt M_0 - M \|$. Note that 
\begin{align*}
	\| \wt M_0 - M \| = \left\| \wt U \wt \Sigma \wt V^{\top} - U \Sigma V^{\top}  \right\|  = \left\| \wt U ( \wt U^{\top} \what L \wt V + E_{\Sigma} )\wt V^{\top} - UU^{\top} M V V^{\top}  \right\|, 
\end{align*}
is a $d_1 \times d_2$ matrix of rank at most $2r$. Since
\begin{equation}\label{decomp_M_0_M}
	\begin{aligned}
		&  \left\| \wt U ( \wt U^{\top} \what L \wt V + E_{\Sigma} )\wt V^{\top} - UU^{\top} M V V^{\top}  \right\|  \leq \left\| \wt U \wt U^{\top} \what L \wt V \wt V^{\top} - UU^{\top}M V V^{\top}  \right\| + \left\| \wt U E_{\Sigma} \wt V^{\top}  \right\| \\
		& \leq \left\| \left( \wt U \wt U^{\top} - UU^{\top} \right) \what L \wt V \wt V^{\top} \right\| + \left\| UU^{\top} \left( \what L - M \right) \wt V \wt V^{\top}  \right\| + \left\| UU^{\top} M \left( \wt V \wt V^{\top} - V V^{\top} \right)  \right\|  + \left\| \wt U E_{\Sigma} \wt V^{\top}  \right\| \\
		& \leq \left\| \wt U \wt U^{\top} - UU^{\top} \right\| \left\| \what L \right\|  + \left\| \what L - M   \right\|  + \|M\| \left\|  \wt V \wt V^{\top} - V V^{\top}\right\|  + \left\| E_{\Sigma} \right\| \\
		& \leq \left\| \wt U \wt U^{\top} - UU^{\top} \right\| \left\| \what L -M \right\| + \left\| \wt U \wt U^{\top} - UU^{\top} \right\| \left\| M \right\|   + \left\| \what L - M   \right\| + \|M\| \left\|  \wt V \wt V^{\top} - V V^{\top}\right\|  + \left\| E_{\Sigma} \right\| \\
		& \leq 2  \left\| \what L - M   \right\| +  \|M\| 
		\left( \left\|  \wt V \wt V^{\top} - V V^{\top}\right\| + \left\| \wt U \wt U^{\top} - UU^{\top} \right\| \right) + \left\| E_{\Sigma} \right\|, 
	\end{aligned}
\end{equation}
it is sufficient to plug in the upper bound (\ref{Delta_bounded_by_bernstein}) of $ \left\| \what L - M \right\|$, $\|M\| = \sigma_1$, as well as the upper bounds (\ref{U_dp_permutation}) of $\|\wt U \wt U^{\top} - UU^{\top}\|$, (\ref{V_dp_permutation}) of $ \|\wt V \wt V^{\top} - VV^{\top}\|$ and (\ref{equ: operator_norm_E_Sigma}) of $\|E_{\Sigma}\|$. In conclusion, conditioned on the event $\calE_*$, as the sample size $$n\geq C_1 C_l^{-1} r \left( \frac{\sqrt{C_u r} \, \sigma_1 + \sigma_{\xi} }{\sigma_r}\right)^2 (d_1\vee d_2) \log (d_1 + d_2), $$ for some absolute constant $C_1>0$, we have  
\begin{align*}
	\| \wt M_0 - M \|
	& \leq C_5 \sqrt{C_l^{-1}} \left( \sqrt{C_u r}  \sigma_1 + \sigma_{\xi} \right)\frac{\sigma_1}{\sigma_r} \sqrt{\frac{(d_1 \vee d_2)\log (d_1 + d_2)}{n}} \\
	& \quad +  C_5  \sqrt{ C_l^{-1}}\left( \sqrt{C_u r} \, \sigma_1 + \sigma_{\xi} \right) \frac{\sigma_1}{\sigma_r}  \frac{\sqrt{r (d_1 \vee d_2)}}{n\varepsilon}  \log n \log^{\frac{1}{2}}(\frac{3.75}{\delta})   \\
	& \quad +  C_5  \sqrt{ C_l^{-1}}\left( \sqrt{C_u r}\sigma_1 + \sigma_{\xi}\right) \frac{r}{n\eps} \log n \log^{\frac{1}{2}}(\frac{3.75}{\delta}), 
\end{align*}
for some absolute constant $C_5>0$ with probability at least $1 - \exp(-d_1) - \exp(-d_2) - 10^{-20r}$. 
\subsection{Proof of Corollary \ref{cor: good_init}} \label{appx: cor_good_init}
The proof of Corollary \ref{cor: good_init} is obtained by setting the upper bound of $ \| \wt M_0 - M \|_F \leq \sqrt{2r} \| \wt M_0 - M \| $ given by Theorem \ref{thm: dp_init} smaller than the order of $\sigma_r$. 

\section{Proof of Theorem~\ref{thm: fano_minimax_lower_bound}} \label{proof_of_thm_fano_minimax_lower_bound}

We first present some preliminary results on the KL-divergence and total variation distance between Gaussian distributions.	Let $\mcN(\mu_1, \Sigma_1)$ and $\mcN(\mu_2, \Sigma_2)$  be two $p$-dimensional multivariate Gaussians, then 
\begin{align*}
	& \operatorname{KL}\left(\mcN\left(\mu_1, \Sigma_1\right) \| \mcN\left(\mu_2, \Sigma_2\right)\right) \\ & = \frac{1}{2}\left(\operatorname{Tr}\left(\Sigma_2^{-1} \Sigma_1 - I_p\right) +\left(\mu_2-\mu_1\right)^{\top} \Sigma_2^{-1}\left(\mu_2-\mu_1\right)+\log \left(\frac{\operatorname{det} \Sigma_2}{\operatorname{det} \Sigma_1}\right)\right). 
\end{align*} 
Let $M_* = \sigma HH^{\top}$ and $M_*^{\prime} = \sigma H^{\prime} H^{\prime \top}$ where $H\neq H^{\prime} \in  \mcS_q^{(d_1+d_2)}$, then 
\begin{align*}
    	 & \operatorname{KL} \left(f_{M_*}(\Bar{y}_i, X_i, X_i^{\prime}) \|  f_{M_*^{\prime}}(\Bar{y}_i, X_i, X_i^{\prime})  \right)\\
      &= \EE_{f_{M_*}(\Bar{y}_i, X_i, X_i^{\prime})}\left[ \log \frac{f_{M_*}(\Bar{y}_i, X_i, X_i^{\prime})}{f_{M_*^{\prime}}(\Bar{y}_i, X_i, X_i^{\prime})} \right] \\
      & = \EE_{X_i, X_i^{\prime}}\EE_{f_{M_*}(\Bar{y}_i\mid X_i, X_i^{\prime})} \left[- \frac{\left(\Bar{y}_i - \left<\left(\begin{array}{cc}
	0 & M \\
	M^{\top} & 0
\end{array}\right), \left(\begin{array}{cc}
	0 & X_i \\
	X_i^{\prime\top} & 0
\end{array}\right)\right>\right)^2}{4\sigma_{\xi}^2} \right]\\ 
 & \quad + \EE_{X_i, X_i^{\prime}}\EE_{f_{M_*}(\Bar{y}_i\mid X_i, X_i^{\prime})} \left[- \frac{\left(\Bar{y}_i - \left<\left(\begin{array}{cc}
	0 & M^{\prime} \\
	M^{\prime\top} & 0
\end{array}\right), \left(\begin{array}{cc}
	0 & X_i \\
	X_i^{\prime\top} & 0
\end{array}\right)\right>\right)^2}{4\sigma_{\xi}^2}\right]\\ 
    & = \EE_{X_i, X_i^{\prime}} \left[ \frac{\left<M_*-M_*^{\prime}, \left(\begin{array}{cc}
	0 & X_i \\
	X_i^{\prime\top} & 0
\end{array}\right)\right> \cdot \left<M_* - M_*^{\prime}, \left(\begin{array}{cc}
	0 & X_i \\
	X_i^{\prime\top} & 0
\end{array}\right)\right>}{4\sigma_{\xi}^2} \right]\\
      & \lesssim \frac{1}{\sigma_{\xi}^2} C_u \|M_* - M_*^{\prime}\|_F^2 \lesssim \frac{1}{\sigma_{\xi}^2} C_u \sigma^2 r \eps_0^2, 
\end{align*}
and further by Pinsker's inequality, we have 
\begin{equation*}
	 \operatorname{TV} \left(f_{M_*}(\Bar{y}_i, X_i, X_i^{\prime}) ,  f_{M_*^{\prime}}(\Bar{y}_i, X_i, X_i^{\prime})  \right) \lesssim  \frac{\sqrt{C_u} \sigma}{\sigma_{\xi}} \sqrt{r}\eps_0.  
\end{equation*}
For any probability measures $P_{M_*}\neq P_{M_*^{\prime}}\in \mcP_{\sigma}$, we have
\begin{equation*}
	 \operatorname{KL} \left(\bigotimes_{i=1}^n f_{M_*}(\Bar{y}_i, X_i, X_i^{\prime}) \|\bigotimes_{i=1}^n f_{M_*^{\prime}}(\Bar{y}_i, X_i, X_i^{\prime})  \right) \lesssim \frac{n}{\sigma_{\xi}^2} C_u \sigma^2 r \eps_0^2.  
\end{equation*}
and 
$$
\sum_{k\in[n]}\operatorname{\operatorname{TV}}\left(f_{M_*}(\Bar{y}_i, X_i, X_i^{\prime}),  f_{M^{\prime}_*}(\Bar{y}_i, (X_i^{\indep})_*) \right) \lesssim n \frac{\sqrt{C_u} \sigma}{\sigma_{\xi}} \sqrt{r}\eps_0, 
$$
The next Lemma \ref{local_packing_set} states that there exists a sufficiently large subsect of $\OO_{d_1+d_2, 2r}$ such that the elements in the subsets are well separated. 

\begin{lem}\label{local_packing_set}
    For any $r\leq d$, there exists a subset $\mcS_q^{(d)}\subset \OO_{d,r}$ with cardinality $\left|\mcS_q^{(d)}\right|\geq 2^{r(d-r)}$ such that for any $U_i\neq U_j\in \mcS_q^{(d)}$,
\begin{align*}
	\|U_iU_i^{\top}-U_jU_j^{\top}\|_q\geq \sqrt{\varepsilon_0^2(1-\varepsilon_0^2)} \|V_i-V_j\|_q\gtrsim \sqrt{\varepsilon_0^2(1-\varepsilon_0^2)}\|V_iV_i^{\top}-V_jV_j^{\top}\|_q\gtrsim \sqrt{\varepsilon_0^2(1-\varepsilon_0^2)} r^{1/q}
\end{align*}
where $\|\cdot\|_q$ denotes the Schatten-q norm and, meanwhile, $$
\|U_iU_i^{\top}-U_jU_j^{\top}\|_{\rm F}\lesssim \|U_i-U_j\|_{\rm F}\leq \varepsilon_0 \|V_i-V_j\|_{\rm F}\leq \sqrt{2r}\varepsilon_0.$$
\end{lem}
By Lemma \ref{local_packing_set}, there exists a subset $\mcS_q^{(d_1+d_2)}\subset \OO_{d_1+d_2,2r}$ with cardinality $\left|\mcS_q^{(d_1+d_2)}\right|\geq 2^{2r(d_1+d_2-2r)}$ such that for any $H\neq H^{\prime}\in \mcS_q^{(d_1+d_2)}$, $$\|HH^{\top}-H^{\prime}H^{\prime \top}\|_q \gtrsim \sqrt{\varepsilon_0^2(1-\varepsilon_0^2)} (2r)^{1/q}, $$ and meanwhile, $$
\|HH^{\top}-H^{\prime}H^{\prime\top}\|_{\rm F}\lesssim 2 \sqrt{r}\varepsilon_0.$$ 

To invoke Lemma~\ref{lem:dp-fano}, we define the metric $\rho: \OO_{d_1+d_2, 2r}\times \OO_{d_1+d_2,2r}\mapsto \RR^+$ as $\rho(H, H^{\prime}):=\|HH^{\top}-H^{\prime}H^{\prime\top}\|_q$ for any $q\in[1,\infty]$ and take  $\rho_0\asymp \tau\varepsilon_0r^{1/q}$,  
$$
l_0 =c_0\frac{n}{\sigma_{\xi}^2} C_u \sigma^2 r \eps_0^2\quad {\rm and}\quad t_0 =  c_0 n \frac{\sqrt{C_u} \sigma}{\sigma_{\xi}} \sqrt{r}\eps_0, 
$$ 
for some small absolute constant $c_0,\tau>0$.  Then, by Lemma~\ref{lem:dp-fano}, for any $(\varepsilon,\delta)$-DP estimator $\wt H$, 
\begin{align*}
	&\sup_{P \in \mcP_{\sigma} } \EE \big\|\wt H\wt H^{\top} - HH^{\top} \big\|_q  \\ 
	& \geqslant \max \left\{\frac{\tau\varepsilon_0 r^{1/q}}{2} \left(1-\frac{  c_0\frac{n}{\sigma_{\xi}^2} C_u \sigma^2 r \eps_0^2+\log 2}{\log N}\right), \frac{\tau \varepsilon_0 r^{1/q}}{4}\left(1 \wedge \frac{N-1}{\exp \left( 4 \varepsilon  c_0 n \frac{\sqrt{C_u} \sigma}{\sigma_{\xi}} \sqrt{r}\eps_0 \right)}\right)\right\}. 
\end{align*}
Recall that $N\geq 2^{2r(d_1+d_2)/2}$ if $d_1+d_2\geq 4r$. We can take 
$$
\varepsilon_0 \asymp \frac{\sigma_{\xi}}{\sqrt{C_u}\sigma}\left (\sqrt{\frac{d_1 + d_2 }{n}}+r^{\frac{1}{2}}\frac{d_1 + d_2 }{n\varepsilon}\right) ,
$$ 
to get
\begin{align*}
 \sup_{P \in \mcP_{\sigma}} \EE\Big\|\wt H \wt H^{\top} - HH^{\top} \Big\|_q \gtrsim \frac{\sigma_{\xi}}{\sqrt{C_u}\sigma}\left (r^{1/q}\sqrt{\frac{d_1+ d_2 }{n}}+r^{\frac{1}{2}+\frac{1}{q}}\frac{d_1+ d_2 }{n\varepsilon}\right)\bigwedge r^{1/q}, 
\end{align*} 
where the last term is due to a trivial upper of $\|\wt H\wt H^{\top}-HH^{\top}\|_q\leq (4r)^{1/q}$. Since $\sigma$ is already known, it suffices to estimate $HH^{\top}$ differentially privately by the estimator $ \wt H \wt H^{\top}$, and an estimator $(\wt M)_*$ for the matrix $M_*$ is given by $ \sigma \wt H \wt H^{\top}$ . Therefore,
\begin{align*}
     \underset{P\in\mcP_{\sigma}}{\sup} \EE\big\|\wt M_* -& M_* \big\|_q 
		\geq  \sigma \cdot \underset{P\in\mcP_{\sigma}}{\sup} \EE \big\|\wt H\wt H^{\top} - HH^{\top} \big\|_q. 
\end{align*}
Due to $\mcP_{\sigma}\subset \mcP_{\Sigma}$, we have 
\begin{align*}
     \underset{P\in\mcP_{\Sigma}}{\sup} \EE\big\|\wt M_* - M_* \big\|_q \geq  \underset{P\in\mcP_{\sigma}}{\sup} \EE\big\|\wt M_* - M_* \big\|_q \gtrsim \frac{\sigma_{\xi}}{\sqrt{C_u}}\left (r^{1/q}\sqrt{\frac{d_1+ d_2 }{n}}+r^{\frac{1}{2}+\frac{1}{q}}\frac{d_1+ d_2 }{n\varepsilon}\right)\bigwedge r^{1/q} \sigma. 
\end{align*}

There is a one-to-one mapping between $\wt M$ and $(\wt M)_*$. Let $\wt M - M = U_{\Delta} \Sigma_{\Delta} V_{\Delta}^{\top}$, then $\|\wt M-M\|_q^q=\left\|\Sigma_{\Delta}\right\|_q^q$. Note that 
$$
(\wt M)_*-M_*=\left(\begin{array}{cc}
	0 & \wt M-M \\
	\wt M^{\top}-M^{\top} & 0
\end{array}\right)=\frac{1}{\sqrt{2}}\left(\begin{array}{ll}
	U_{\Delta} & U_{\Delta} \\
	V_{\Delta} & -V_{\Delta}
\end{array}\right)\left(\begin{array}{cc}
	\Sigma_{\Delta} & 0 \\
	0 & \Sigma_{\Delta}
\end{array}\right) \frac{1}{\sqrt{2}}\left(\begin{array}{ll}
	U_{\Delta} & U_{\Delta} \\
	V_{\Delta} & -V_{\Delta}
\end{array}\right)^{\top}, 
$$
and 
$$
\left\|(\wt M)_*-M_*\right\|_q^q=\left\|\left(\begin{array}{cc}
	\Sigma_A & 0 \\
	0 & \Sigma_a
\end{array}\right)\right\|_q^q=2\left\|\Sigma_{\Delta}\right\|_q^q=2\|\wt M-M\|_q^q. 
$$
Therefore, $\left\|(\wt M)_*-M *\right\|_q=2^{1 / q}\|\wt M-M\|_q$ and 

\begin{align}
 \inf_{\wt M} \underset{P\in\mcP_{\Sigma}}{\sup} \EE\big\|\wt M -& M \big\|_q \gtrsim \frac{\sigma_{\xi}}{\sqrt{C_u}}\left (r^{1/q}\sqrt{\frac{d_1+ d_2 }{n}}+r^{\frac{1}{2}+\frac{1}{q}}\frac{d_1+ d_2 }{n\varepsilon}\right)\bigwedge r^{1/q} \sigma,
\end{align} 
where we use the fact $d_1+d_2 \lesssim d_1 \vee d_2$ and infimum is taken over all possible $(\varepsilon,\delta)$-DP algorithms.

\section{Proof of Theorem \ref{thm: convergence_rate}} \label{appx: proof_of_thm_convergence_rate}

In Appendix \ref{appx: proof_of_thm_convergence_rate}, we aim to prove Theorem \ref{thm: convergence_rate}. The proof is composed of three Parts. In Part \ref{part_1_thm_convergence_rate}, we characterize the sensitivity $\Delta_l$ for iterations $l = 1, \cdots, l^*$ and bound $\|\calP_{\TT_l} N_l\|_F$. In Part \ref{part_2_thm_convergence_rate}, take mathematical induction to prove that if the RIP-condition holds and both $\Delta_l$ and $\|M_l - M\|_F$ are bounded with high probability, then we also have $\Delta_{l+1}$ and $\|M_{l+1} - M\|_F$ are bounded with high probability. In Part \ref{part_3_thm_convergence_rate}, we choose an appropriate $l^*$ as the total number of iterations and give the convergence rate of $\|\wt M_{l^*} - M\|_F$. 

\subsection{Part 1} \label{part_1_thm_convergence_rate}
In Part \ref{part_1_thm_convergence_rate}, we focus on upper bounding $ \| \mathcal{P}_{\mathbb{T}_l} N_l \|_F $. The first step is to characterize the sensitivity of the $l$-th iteration for $l\in[l^*]$. Let $$G_l^{(i)} := \frac{1}{n}\sum_{j\neq i}\left( \left< X_j, M_l \right> - y_j \right) X_j + \frac{1}{n}\left[ \left< X_i^{\prime}, M_l\right> - y_i^{\prime} \right]X_i^{\prime}, $$ which is the gradient of $l$-th iteration obtained by the dataset differes with the original one only by the $i$-th pair of observation. The sensitivity of the gradient descent on the tagent space $\TT_l$ is
\begin{align*}
	\Delta_{l} & :=  \max_{{\rm neighbouring} (Z, Z^{\prime})} \left\| M_l - \eta \calP_{\TT_l}(G_l(Z)) - \left[  M_l - \eta \calP_{\TT_l}(G_l(Z^{\prime})) \right] \right\|_F \\ &= \max_{i\in[n]} \left\| M_l - \eta \calP_{\TT_l}(G_l) - \left[  M_l - \eta \calP_{\TT_l}(G_l^{(i)}) \right] \right\|_F. 
\end{align*}
By the definition of $G_l$ and $G_l^{(i)}$,  
\begin{align*}
	\Delta_{l} \leq \frac{\eta}{n} \max_{i\in[n]}  \left[ \left\| \calP_{\TT_l} \left( \left< X_i, M_l \right> - y_i \right) X_i \right\|_F + \left\| \calP_{\TT_l} \left( \left< X_i^{\prime}, M_l \right> - y_i^{\prime} \right) X_i^{\prime} \right\|_F \right], 
\end{align*}
where for all $i\in[n]$ and $l+1 \in [l^*]$
\begin{equation}\label{single_term_in_GD_sensitivity}
	\begin{aligned}
		\left\| \calP_{\TT_l} \left( \left< X_i, M_l \right> - y_i \right) X_i \right\|_F & \leq \left| \left< X_i, M_l - M \right> - \xi_i  \right| \left\| \calP_{\TT_l}  X_i \right\|_F \\
		& \leq \left( |\xi_i| + \left| \left<X_i, M_l - M\right> \right| \right) \sqrt{2r} \left\| X_i \right\|. 
	\end{aligned}
\end{equation}
Here, the last inequality uses the fact that for any $B\in\RR^{d_1 \times d_2}$, the matrix $\calP_{\TT_l}B$ is of rank at most $2r$. Since both $\xi_i$ and $\left<M_l-M, X_i\right> $ are sub-Gaussians with  $\|\xi_i\|_{\psi_2} = \sigma_{\xi}$ and $\left\|\left<M_l-M, X_i\right> \right\|_{\psi_2} \leq \left\| 
M_l - M \right\|_F \sqrt{C_u}$, we turn to Lemma \ref{lem: sub-Gaussian} to upper bound $|\xi_i| + \left| \left<X_i, M_l - M\right> \right| $. \begin{lem}[\cite{vershynin2018high}] \label{lem: sub-Gaussian}
    For any sub-Gaussian random variable $B\in \RR$, 
    $$ \mathbb{P}(|B| \geq t) \leq 2 \exp \left(-c t^2 /\|B\|_{\psi_2}\right), \quad \forall t \geq 0.  $$
\end{lem}
According to the tail probability of sub-Gaussian random variable stated in Lemma \ref{lem: sub-Gaussian}, we have with probability at least $1-n^{-10}$, 
\begin{align} \label{xi_plus_regression}
	|\xi_i| + \left| \left<X_i, M_l - M\right> \right| \leq C_1 (\sigma_{\xi}+ \left\| M_l -M \right\|_F\sqrt{C_u})\log^{1/2} n, 
\end{align}
for some absolute constant $C_1 > 0$. \cite{shen2023computationally} offeres the folowing result on $\|X_i\|$. 
\begin{lem}[\cite{shen2023computationally}] \label{lem: norm_of_X}
	Suppose the vectorization of $X \in \mathbb{R}^{d_1 \times d_2}$ follows mean zero multivariate Gaussian distribution $N(\mathbf{0}, \Lambda)$ where $\Lambda \in \mathbb{R}^{d_1 d_2 \times d_1 d_2}$ satisfies $\lambda_{\max }(\Lambda) \leq C_u$. Then, for some constant $c>0$ 
	$$
	\mathbb{P}\left(\|X\| \geq t+c \sqrt{C_u\left(d_1+d_2\right)}\right) \leq \exp \left(-t^2\right). 
	$$
	It implies $\left\|\| X \|\right\|_{\psi_2} \leq c_1 \sqrt{C_u(d_1+d_2)}$ and $\left\|\| X \|\right\|_{\psi_1} \leq c_2 \sqrt{C_u (d_1+d_2)}$ for some constants $c_1, c_2>0$.
\end{lem}
Thus, for some absolute constant $C_2>0$, with probability at least $1- \exp\left(-10 C_u(d_1+d_2)\right)n^{-10}$ 
\begin{align} \label{norm_of_X_i}
	\|X_i\| & \leq C_2 \sqrt{C_u (d_1+d_2)+\log n}. 
\end{align}
Combining (\ref{single_term_in_GD_sensitivity}), (\ref{xi_plus_regression}) and (\ref{norm_of_X_i}) and taking maximum over $n$, for some constant $C_3>0$, we have the event
\begin{equation}\label{event_Delta_l_prime}
	\calE_{\Delta_l}^{\prime}:= \left\{ \Delta_l \leq C_3 \frac{\eta}{n} (\sigma_{\xi}+ \|M_l -M\|_F\sqrt{C_u}) \sqrt{C_u r (d_1+d_2+\log n)\log n} \right\}, 
\end{equation}
happens with probability at least $1-n^{-9}-\exp\left( -10 C_u(d_1+d_2) \right)n^{-9}$. In the event $\calE_{\Delta_l}^{\prime}$ stated in (\ref{event_Delta_l_prime}), the sensitivity $\Delta_l$ still relies on $\|M_l - M\|_F$. To get an upper bound irrelevant with $l$, we take condition on the event $$\calE_l =  \left\{ \left\| M_{l} -M \right\|_F \leq c_0 \sigma_r \right\}, $$ and obtain that for some absolute constant $\wt C_3>0$, the event 
\begin{equation} \label{event_Delta_l}
	\calE_{\Delta_l} := \left\{ \Delta_l \leq \wt C_3 \frac{\eta}{n} (\sigma_{\xi}+ \sigma_r \sqrt{C_u}) \sqrt{C_u r (d_1+d_2+\log n)\log n} \right\}, 
\end{equation}
happens with the probability $\PP(\calE_{\Delta_l}) \geq 1-n^{-9}-\exp\left( -10 C_u(d_1+d_2) \right)n^{-9}. $

In the $l+1$-th iteration of Algorithm \ref{alg:DP-RGrad}, the matrix $M_{l}$ and operator $\calP_{\TT_l}$ are known. Moreover, the rank $r$ approximation $\SVDr$ is irrelevant with the data set $Z = \{(X_i, y_i)\}_{i=1}^n$. Thanks to the post-processing property and composition property of differential privacy, we only need to guarantee that $M_l - \eta_l \calP_{\TT_l}(G_l)$ is $(\eps / l^*, \delta / l^*)$-DP where the gradient $$ G_l = \frac{1}{n} \sum_{i=1}^n \left( \left<X_i, M_l\right> - y_i \right)X_i, $$ is the only component depends on the data set $Z$. Let $N_l$ be a $d_1\times d_2$ matrix with entries i.i.d. to normal distribution with varicance $ \frac{l^{*2}\Delta_l^2}{\eps^2}\log\left( \frac{1.25 l^*}{\delta} \right). $
Under the condition that $d_1+d_2\gtrsim \log n$ and conditioned on the event $\calE_{\Delta_l} \cap \calE_{l}$, we have for some constant $C_4>0$ 
\begin{align*}
	\| \calP_{\TT_l}N_l \|_F \leq C_4 \eta l^* \frac{ r (d_1+d_2)}{n\eps} (\sigma_{\xi}+ \sigma_r \sqrt{C_u}) \sqrt{C_u \log n} \log^{1/2}\left( \frac{1.25 l^*}{\delta} \right), 
\end{align*}
with probability at least $1-\exp(-(d_1+d_2))$.  

\subsection{Part 2} \label{part_2_thm_convergence_rate}
In Part \ref{part_2_thm_convergence_rate}, we take mathematical induction to prove that when the events $$\calE_l^*:=  \calE_{{\rm RIP}} \cap \calE_{l}\cap \calE_{\Delta_l} , $$ occurs with high probability, the event $ \calE_{l+1}^*$ occurs with high probability as well. Here, $\calE_{{\rm RIP}}$ is defined as the event where the RIP condition of $\{X_i\}_{i\in[n]}$ holds, See (\ref{event_RIP}). 

\noindent\emph{Step 1: $\calE_{0}^*$ is true with high probability. }

We first consider the RIP condition. According to Lemma \ref{lem: RIP_condition}, for any $B\in \RR^{d_1\times d_2}$ of rank $r$, there exist constants $c_1, c_2, c_3>0$ and $0<c_4<c_5$ such that when $n \geq c_1 r(d_1 + d_2)$, with probability at least $1-c_2 \exp \left(-c_3 r(d_1+d_2)\right)$,
$$
(1-R_r)\|B\|_{\mathrm{F}}^2 \leq \frac{1}{n} \sum_{i=1}^n\left\langle X_i, B\right\rangle^2 \leq (1+R_r)\|B\|_{\mathrm{F}}^2, 
$$ 
where $R_r := \left(1- c_4 \sqrt{C_u C_l}\right) \vee \left( c_5 \sqrt{C_u C_l} -1\right)$. The values of $R_{2r}, R_{3r}$ and $R_{4r}$ are defined similarly. Therefore, under the condition that $n \geq \wt c_1 r(d_1 + d_2)$, for some constants $\wt c_1, \wt c_2, \wt c_3, \wt c_4, \wt c_5 > 0$, 

\begin{equation}\label{event_RIP}
	\calE_{{\rm RIP}}:= \left\{R_r\vee R_{2r}\vee R_{3r}\vee R_{4r} \leq  \left(1- \wt c_4 \sqrt{C_u C_l}\right) \vee \left( \wt c_5 \sqrt{C_u C_l} -1\right) \right\}, 
\end{equation}
happens with probability at least $1-\wt c_2 \exp \left(-\wt c_3 r(d_1+d_2)\right)$. 

As for $\calE_{0}$, we refer to Corollary \ref{cor: good_init}, which shows that as the sample size $$n \geq \wt O \left( (\kappa^4 r^2 + \kappa^2 \kappa_{\xi}^2r) (d_1\vee d_2) \right),$$ the event $\calE_{0}$ happens with probability at least $1- (d_1+d_2)^{-10} - n^{-9} - \exp(-d_1) - \exp(-d_2) - 10^{-20r}$. Conditioned on $\calE_0$, plugging $l=0$ to the event $\calE_{\Delta_l}^{\prime}$ defined in (\ref{event_Delta_l_prime}), we have the event $\calE_{\Delta_l}$ defined in (\ref{event_Delta_l}) happens with probability at least $1-n^{-9}-\exp \left(-10 C_u\left(d_1+d_2\right)\right) n^{-9}$. To this end, we have  
\begin{align*}
	\PP\left( \calE_0^* \right) & = \PP\left( \calE_{{\rm RIP}}\cap \calE_0 \cap \calE_{\Delta_0} \right)  \geq 1 -\wt c_2 \exp \left(-\wt c_3 r(d_1+d_2)\right)\\
	& \quad \quad  -(d_1+d_2)^{-10}-n^{-9} - e^{-d_1} - e^{-d_2} - 10^{-20r}\\
	& \quad \quad  - n^{-9} -\exp\left( -10C_u(d_1+d_2) \right)n^{-9}. 
\end{align*} 
\noindent\emph{Step 2: induction. }
The following analysis is conditioned on the event $\calE_{l}^*$. Let $\calX: \RR^{d_1\times d_2} \rightarrow \RR^n$ be an operator defined by 
$\calX(B) = \left( \left< X_1, B \right>, \cdots, \left< X_n, B \right> \right)^{\top} \in \RR^n ,$ for all $B\in \RR^{d_1 \times d_2}$. It is easy to check that the adjoint operator of $\calX$ is $\calX^*: \RR^n \rightarrow \RR^{d_1 \times d_2}$ which is defined by $ \calX^* (b): = \sum_{i=1}^n b_i X_i, $ for all $b\in \RR^{n}$. Therefore, $\calX^* \calX (M_l) = \sum_{i=1}^n \left< X_i, M_l \right> X_i \quad {\rm and}\quad \calX^* (\xi) = \sum_{i=1}^n \xi_i X_i, $ where $\xi := (\xi_1, \cdots, \xi_n) \in \RR^n$ and accordingly, $$P_{\TT_l}G_l = \frac{1}{n} \left[ \calP_{\TT_l}\calX^* \calX \left( M_l -M \right) - \calP_{\TT_l} \calX(\xi) \right]. $$
Our first goal is to upper bound 
\begin{equation} \label{decompose_gradient_on_tangent_space}
	\begin{aligned}
		& \| M_l - M - \eta \calP_{\TT_l}(G_l) \|_F 
		=  \| M_l - M - \frac{\eta}{n}  \calP_{\TT_l}\calX^* \calX \left( M_l -M \right) - \calP_{\TT_l} \calX(\xi) \|_F \\
		& \leq \left(1-\frac{\eta}{n}\right)\| M_l - M \|_F + \underbrace{\frac{\eta}{n} \left\|  
			\left(\mathcal{I} - \calP_{\TT_l} \calX^* \calX \calP_{\TT_l} \right) \left( M_l - M \right)\right\|_F}_{D_1}  \\
		& \quad + \underbrace{\frac{\eta}{n} \left\| \calP_{\TT_l} \calX^* \calX \calP_{\TT_l^{\perp}} \left( M_l - M \right)\right\|_F}_{D_2} + \underbrace{\frac{\eta}{n} \left\| \calP_{\TT_l} \calX^* \left( \xi \right)\right\|_F}_{D_3}. 
	\end{aligned} 
\end{equation}
Lemma \ref{lem: element_based on_event_RIP} characterizes the operators $\calP_{\TT_l} - \calP_{\TT_l}\calX^* \calX \calP_{\TT_l}$ and $\calP_{T_l} \calX^* \calX \calP_{\TT_l^{\perp}}$, which are critical to upper bound $D_1$ and $D_2$ in (\ref{decompose_gradient_on_tangent_space}). 
\begin{lem}[\cite{wei2016guarantees}, \cite{luo2022tensor}] \label{lem: element_based on_event_RIP}
	Suppose the event $\calE_{{\rm RIP}}$ happens, then the following conclusions hold 
	\begin{enumerate}
		\item $\left\| \calP_{\TT_l} - \calP_{\TT_l}\calX^* \calX \calP_{\TT_l} \right\| \leq R_{2r}$. 
		\item $\|\calP_{T_l} \calX^* \calX \calP_{\TT_l^{\perp}} \left( M_l -M \right) \|_F = R_{4r} \left\| \calP_{\TT_l^{\perp}} (M_l -M) \right\|_F$, 
	\end{enumerate}
	where $\|\calP_{\TT_l^{\perp}} \left( M_l - M \right) \|_F \leq \frac{1}{\sigma_r} \|M_l -M\| \|M_l - M\|_F$ according to \cite{wei2016guarantees}.
\end{lem}
According to Lemma \ref{lem: element_based on_event_RIP}, conditioned on the event $\calE_l^*$
\begin{align*}
	& D_1 = \frac{\eta}{n} \left\| \mathcal{I} - \calP_{\TT_l} \calX^* \calX \calP_{\TT_l} \left( M_l -M \right) \right\|_F \\
	& \leq   \frac{\eta}{n} \left[  \left\| \calP_{\TT_l}  - \calP_{\TT_l} \calX^* \calX \calP_{\TT_l} \left( M_l -M \right) \right\|_F  + \| \calP_{\TT_l^{\perp}} \left( M_l - M \right) \|_F \right] \leq  \frac{\eta}{n} \left( R_{2r} + c_0 \right) \left\| M_l -M \right\|_F, 
\end{align*}
and $	D_2 = \frac{\eta}{n} \| \calP_{T_l} \calX^* \calX \calP_{\TT_l} \left( M_l -M \right) \|_F \leq \frac{\eta}{n} R_{4r} c_0 \left\| M_l -M \right\|_F. $
To this end, the only term unknown in (\ref{decompose_gradient_on_tangent_space}) is $$ D_3 = \left\| \calP_{\TT_l} \calX^* \left( \xi \right)\right\|_F =  \left\| \calP_{\TT_l} \sum_{i=1}^n \xi_i X_i  \right\|_F \leq \sqrt{2r} \left\| \sum_{i=1}^n \xi_i X_i \right\|, $$
where $\left\| \sum_{i=1}^n \xi_i X_i \right\|$ is the spectral norm of a summation of $n$ i.i.d. mean zero sub-exponential random matrices. We upper bound  $\left\| \sum_{i=1}^n \xi_i X_i \right\|$ by Theorem \ref{thm: bernstein}. Let $B_i := \xi_i X_i$ for all $i=1, \cdots, n$, then
\begin{align*}
	K:= \max_{i\in[n]}  \left\| \|B_i\| \right\|_{\psi_1} =  \max_{i\in[n]}  \left\| \|\xi_i X_i\| \right\|_{\psi_1} \leq \max_{i\in[n]}  \| \xi_i \|_{\psi_2} \left\| \|X_i\| \right\|_{\psi_2} \leq \sqrt{C_u(d_1+d_2)\sigma_{\xi}^2}.
\end{align*} 
Since $\EE \xi_i^4 \leq 3 \sigma_{\xi}^4$ and $ \EE \|X_i\|^4 \leq 3 C_u (d_1+d_2)^2 $, 
$$ \|\EE B_i B_i^{\top}\| = \EE \xi_i^2   \EE \|X_i\|^2 \leq \frac{C_u (d_1+d_2)}{2 \sigma_{\xi}^2} \EE \xi_i^4 + \frac{\sigma_{\xi}^2}{2 C_u (d_1+d_2)} \EE \|X_i\|^4 \leq 3 \sigma_{\xi}^2 C_u (d_1+d_2), $$ where the first inequality uses the fact $a b \leq \frac{a^2}{2 c}+\frac{c b^2}{2}$. 
Similarly, $ \|\EE B_i B_i^{\top}\| \leq 3 \sigma_{\xi}^2 C_u (d_1+d_2)$.

Therefore, $$S^2: = \|\EE B_i B_i^{\top}\|  \vee \|\EE B_i^{\top} B_i\| \leq 3 \sigma_{\xi}^2 C_u (d_1+d_2).$$ Applying Thorem \ref{thm: bernstein} with $\alpha=1$, $ K = c_1 \sqrt{C_u(d_1+d_2)\sigma_{\xi}^2}$ and $S = \sqrt{3 \sigma_{\xi}^2 C_u (d_1+d_2)}$, we have
\begin{align*}
	\PP\left( \frac{1}{n} \left\| \sum_{i=1}^n \xi_i X_i \right\| \geq C_5 \sqrt{C_u (d_1+d_2) \sigma_{\xi}^2} \sqrt{\frac{\log(d_1+d_2)}{n}} \right) \leq (d_1+d_2)^{-10}. 
\end{align*}
In conclusion, with probability at least $1- (d_1+d_2)^{-10}$
\begin{align*}
	D_3 \leq C_1 \sqrt{r} \left\|\sum_{i=1}^n \xi_i X_i \right\|_F \leq C_1 \sqrt{C_u \sigma_{\xi}^2 n r (d_1+d_2) \log(d_1+d_2)}, 
\end{align*}
for some constant $C_1 > 0$. 

Conditioned on $\calE_l^*$, we plug $D_1$, $D_2$ and $D_3$ into (\ref{decompose_gradient_on_tangent_space}) and obtain that for some small constant $0<c_0<1$ and absolute constant $C_2 >0$ 
\begin{align*}
	& \| M_l - M - \eta \calP_{\TT_l}(G_l) + \calP_{\TT_l}N_l \|_F \leq   \| M_l - M - \eta \calP_{\TT_l}(G_l)\|_F + \|  \calP_{\TT_l}N_l \|_F \\
	& \leq \left(1- \rho_0 \right)\| M_l - M \|_F  +  C_2 \eta \sigma_{\xi} \sqrt{ C_u}  \sqrt{\frac{ r (d_1+d_2) }{n}}  \log^{1/2}(d_1+d_2) \\
	& \quad +  C_2 \eta l^*    \sqrt{C_u }(\sigma_{\xi} + \sigma_r \sqrt{C_u}) \frac{ r (d_1+d_2) }{n\eps} \log^{1/2} n \log^{1/2}\left( \frac{1.25 l^*}{\delta} \right),  
\end{align*} 
with probability at least $1- (d_1+d_2)^{-10} - \exp(-(d_1+d_2))$ where we define $$\rho_0 := \frac{\eta}{n} \left( 1-R_{2r} - c_0 - R_{4r} c_0 \right). $$ 

Suppose that $c_0 \lesssim \frac{1}{R_{2r}(1+R_{4r})} \wedge \frac{1}{8}$, the step size $\eta\leq n$ being a small constant, then we have $0\leq \rho_0<1$. Further, as for some absolute constant $C_3>0$, the sample size satisfies
\begin{align*}
	n \geq C_3 \max \Bigg\{ & \eta^2  \left(\frac{\sigma_{\xi}}{\sigma_r}\right)^2 C_u r (d_1+d_2) \log (d_1+d_2), \\
	& \eta l^*    \sqrt{C_u } \left(\frac{\sigma_{\xi} + \sigma_r \sqrt{C_u}}{\sigma_r}\right) \frac{ r (d_1+d_2) }{\eps} \log^{1/2} n \log^{1/2}\left( \frac{1.25 l^*}{\delta}\right) \Bigg\}, 
\end{align*}
we have for some small constant $0<\rho_1<1$, $$\| M_l - M - \eta \calP_{\TT_l}(G_l) + \calP_{\TT_l}N_l \|_F \leq \left(1-\rho_1\right) \| M_l - M \|_F \leq (1-\rho_1)c_0\sigma_r.  $$ 
Applying Lemma \ref{lem: matrix_permutation}, we obtain that under the condition $ 40 c_0< \frac{\rho_1}{2} $, 
\begin{equation} \label{convergence_1}
	\begin{aligned}
		\left\|M_{l+1} - M\right\|_F & \leq \left( 1 + 40 c_0 \right) \left\|  M_l - M - \eta \calP_{\TT_l}(G_l) + \calP_{\TT_l}N_l  \right\|_F \\
		& \leq  \left( 1 + 40 c_0 \right)  (1-\rho_1)  \| M_l - M \|_F\\
		& \leq \left( 1 - \frac{\rho_1}{2} \right)   \| M_l - M \|_F \leq \left( 1 - \frac{\rho_1}{2} \right) c_0 \sigma_r.
	\end{aligned}
\end{equation}
In summary, conditioned on the event $\calE_l^* $, the event  
$$\calE_{l+1}:= \left\{ \left\| M_{l+1} -M \right\|_F \leq c_0 \sigma_r \right\}, $$
occurs with probability at least $1- (d_1+d_2)^{-10} - \exp(-(d_1+d_2))$. Besides, according to $\calE_{\Delta_l}^{\prime}$ defined in (\ref{event_Delta_l_prime}), the event 
$$\mathcal{E}_{\Delta_{l+1}}:=\left\{\Delta_{l+1} \leq \widetilde{C}_3 \frac{\eta}{n}\left(\sigma_{\xi}+\sigma_r \sqrt{C_u}\right) \sqrt{C_u r\left(d_1+d_2+\log n\right) \log n}\right\}, $$
occurs with probability $1-n^{-9}-\exp \left(-10 C_u\left(d_1+d_2\right)\right) n^{-9}$. Therefore, conditioned on $\calE_{l}^*$, the event $\calE_{l+1}^*$ happens with probability at least $1- (d_1+d_2)^{-10} - \exp(-(d_1+d_2))-n^{-9}-\exp \left(-10 C_u\left(d_1+d_2\right)\right) n^{-9}. $

To this end, we has finished the induction and conclude Part \ref{part_2_thm_convergence_rate} by 
\begin{align*}
	\PP\left( \bigcap_{i=0}^l  \calE_{i}^* \right) 
	& \geq 1 -\wt c_2 \exp \left(-\wt c_3 r(d_1+d_2)\right) - (d_1+d_2)^{-10}-n^{-9} - e^{-d_1} - e^{-d_2} - 10^{-20r} \\ 
	& \quad - l \left( (d_1+d_2)^{-10} + \exp(-(d_1+d_2)) \right) - (l+1) \left( n^{-9} + \exp\left( - 10 C_u(d_1+d_2) \right)n^{-9} \right).  
\end{align*}

\subsection{Part 3} \label{part_3_thm_convergence_rate}
In Part \ref{part_3_thm_convergence_rate}, we derive the convergence rate of $\left\|M_{l^*} - M\right\|_F$ and choose an appropriate value for $l^*$. Conditioned on the event $\bigcap_{i=0}^{l^* -1}  \calE_{i}^*$, according to (\ref{convergence_1}), with probability at least $ 1- (d_1+d_2)^{-10} - \exp(-(d_1+d_2)) $ 
\begin{align*}
	& \left\|\wt M_{l^*} - M\right\|_F = \left\|M_{l^*} - M\right\|_F \\
	& \leq \left(1- \rho_0 \right)^{l^*}\| M_0 - M \|_F + \left( \sum_{l=0}^{l^*-1} \left(1- \rho_0 \right)^{l^*-l -1}  \right) C_2 \eta \sigma_{\xi} \sqrt{ C_u}  \sqrt{\frac{ r (d_1+d_2) }{n}}  \log^{1/2}(d_1+d_2) \\
	& \quad +  \left( \sum_{l=0}^{l^*-1} \left(1- \rho_0 \right)^{l^*-l -1}  \right)  C_2 \eta l^*    \sqrt{C_u }(\sigma_{\xi} + \sigma_r \sqrt{C_u}) \frac{ r (d_1+d_2) }{n\eps} \log^{1/2} n \log^{1/2}\left( \frac{1.25 l^*}{\delta} \right). 
\end{align*}

Let $\left\|M_0-M^*\right\|_F=c_0^*$ and $l^* := \log \left(c_0^* n\right) /\rho_0 $, then we have $\left(1- \rho_0 \right)^{l^*}\| M_0 - M \|_F \asymp \frac{1}{n}$, indicating that there is little reason to run the algorithm further than $O(\log n)$ iterations. 

In conclusion, 
\begin{align*}
	\left\|\wt M_{l^*} - M\right\|_F & \leq  C_3 \sigma_{\xi} \sqrt{ C_u}  \sqrt{\frac{ r (d_1+d_2) }{n}}  \log^{1/2}(d_1+d_2) \\
	& + C_3 \sqrt{C_u }(\sigma_{\xi} + \sigma_r \sqrt{C_u}) \frac{ r (d_1+d_2) }{n\eps} \log^{1/2} n \log^{3/2}\left( \frac{1.25 l^*}{\delta} \right) . 
\end{align*} 
with probability at least 
\begin{align*}
	1  &  -\wt c_2 \exp \left(-\wt c_3 r(d_1+d_2)\right) -(d_1+d_2)^{-10}-n^{-9} - e^{-d_1} - e^{-d_2} - 10^{-20r} \\ 
	& - \left( (d_1+d_2)^{-10} + \exp(-(d_1+d_2))+ n^{-9} + \exp\left( - 10 C_u(d_1+d_2) \right)n^{-9} \right)\log n. 
\end{align*} 

\section{The lower bound derived by score attack argument} \label{sec: score_the_lower_bound}
Let $\MM_r:= \{M\in \RR^{d_1\times d_2}: \operatorname{rank}(M) = r\}$. This section establishes the minimax lower bound of differentially privately estimating the matrix $M\in \bigcup_{k=1}^{r} \MM_k, $ within the trace regression model based on an alternative approach, score attack argument \cite{cai2023score}.

The score attack argument involves designing a test statistic and establishing the lower bound of the statistic with the help of a prior distribution of the parameters to estimate. It is unclear, however, how to construct a prior distribution for the low-rank matrix $M$ such that the prior complies with the parameter space $\MM_r$ and the \textit{score attack} is easy to compute at the same time. Compared to DP-fano's Lemma (See Lemma \ref{lem:dp-fano}) which requires $\delta \lesssim e^{-n}$, the score attack argument is valid for a wider range of $\delta \lesssim n^{1+\gamma}$ where $\gamma>0$ is a constant. We first define some necessary notations for the elaboration of score attack argument.  For any matirx $B, C\in\RR^{d_1\times d_2}$, we denote $\supp(B):=\{ (i, j)\in[d_1]\times[d_2]: B_{ij}\neq 0\}$ as the support of $B$ and the matrix $C$ restricted on $\supp B$ is $\left[ C\right]_{\supp(B)} = \sum_{i=1}^{d_1}\sum_{j=1}^{d_2} C_{ij}e_i e_j^{\top} \II (B_{ij}\neq 0)$ where $e_i$ is the $i$-th canonical basis in $\RR^{d_1}$ and $e_j$ is the $j$-th canonical basis in $\RR^{d_2}$. 

To apply score attack argument, we relax the problem to deriving minimax lower bounds over $\MM_{r, d_1} := \{ M\in \RR^{d_1\times d_2}: \supp(M) \subset [d_1] \times [r]\}\subset \bigcup_{k=1}^{r} \MM_k$. The benefit is that there exists a trivial prior of $M\in \MM_{r, d_1}$ such that $M_{ij}\stackrel{i.i.d.}{\sim} \calN(0,1)$ for $ (i, j) \in [d_1] \times [r]$ and $M_{ij} = 0$ otherwise. Similarly, we may consider establish minimax lower bound over $\MM_{r, d_2} := \{ M\in \RR^{d_1\times d_2}: \supp(M) \subset [r] \times [d_2] \}\subset \bigcup_{k=1}^{r} \MM_k. $ For any $M\in \MM_{r, d_2}$, there is a trivial prior as well. Let $A$ be a randomized algorithm mapping a dataset $Z$ to a $d_1\times d_2$ matrix. We define the DP-constrained minimax risk over $\bigcup_{k=1}^{r} \MM_k$ as $$\operatorname{risk}(\bigcup_{k=1}^{r} \MM_k):=\inf_{A}\sup_{M\in\MM_r} \EE  \left\|A(Z)-M\right\|_F^2, $$ where $A$ is taken over all $(\epsilon, \delta)$-DP algorithms. Similarly, we define $\operatorname{risk}(\MM_{r, d_1})$ and $\operatorname{risk}(\MM_{r, d_2})$. Since $\MM_{r, d_1}\subset \bigcup_{k=1}^{r} \MM_k$ and $\MM_{r, d_2}\subset \bigcup_{k=1}^{r} \MM_k$, we have 
\begin{equation}\label{equa: risk_relation}
	\operatorname{risk}(\bigcup_{k=1}^{r} \MM_k) \geq \operatorname{risk}(\MM_{r, d_1}) \bigvee \operatorname{risk}(\MM_{r, d_2}), 
\end{equation}
which indicates that the lower bound of $ \operatorname{risk}(\bigcup_{k=1}^{r} \MM_k)$ will be an immediate result once we successfully lower bound $ \operatorname{risk}(\MM_{r, d_1})$ and $\operatorname{risk}(\MM_{r, d_2})$. 

Next, we construct \textit{score attacks} to derive the lower bounds of $ \operatorname{risk}(\MM_{r, d_1})$ and $\operatorname{risk}(\MM_{r, d_2})$. Let $z=(X, y)$ be the pair of a measurement matrix and its corresponding response variable, drawn independently from (\ref{dist: prob_trace_model}). The score function is defined by
\begin{align*}
	S_M(z):= \nabla_{M}\log f_M(z) = \nabla_{M}\log f_M(y|X), 
\end{align*}
and the score attack is defined by
$$\calA^{(1)}_M(z, A(Z)):= \left< \left[ A(Z)-M\right ]_{[d_1] \times [r]}, S_M(z)\right>, $$
where $A$ is an $(\eps, \delta)$-DP algorithm to estimate $M\in \MM_r$; $z = (X, y)$ is a piece of datum that we want to test whether it belongs to $Z=\{(X_i, y_i)\}_{i=1}^n$;   the quantity $\left[ A(Z)-M\right ]_{[d_1] \times [r]}$ is obtained by restricting $A(Z)-M\in \RR^{d_1\times d_2}$ to the index set $[d_1] \times [r]$. Under some regularity conditions, the score attack $\calA^{(1)}_M(z, A(Z))$ will lead to the lower bound of $\operatorname{risk}(\MM_r, d_1)$. Similarly, we derive the lower bound of $ \operatorname{risk}(\MM_r, d_2)$ with the help of the attack $$\calA^{(2)}_M(z, A(Z)):= \left< \left[ A(Z)-M\right ]_{[r] \times [d_2]}, S_M(z)\right>. $$

Finally, Theorem \ref{thm: minimax_lower_bound} establishes the lower bound for estimating the low-rank matrix $M$. The proof of Theorem \ref{thm: minimax_lower_bound}. 

\begin{thm} \label{thm: minimax_lower_bound}
	Consider i.i.d. observations $Z = \{z_1, \cdots, z_n\}$ drawn from the trace regression model defined in $(\ref{equ: trace_regression_model})$, where $z_i := (X_i, y_i)$ for $i=1, \cdots, n$. We assume that $\{X_i\}_{i\in[n]}$ satisfy the Assumption \ref{assumption_for_X}, $r(d_1 \vee d_2) \lesssim n\eps$, $0<\eps< 1$ and $\delta \lesssim n^{-(1+\gamma)}$ for some $\gamma >0$, then 
	\begin{equation} \label{equ: lower_bound_trace_regression_model}
		\operatorname{risk}(\bigcup_{k=1}^{r} \MM_k) = \inf_{A}\sup_{M\in \bigcup_{k=1}^{r} \MM_k} \EE \left\| A(Z) - M \right\|_F^2 \gtrsim \underbrace{\sigma_{\xi}^2 \;  \frac{r(d_1\vee d_2)}{n}}_{a_1} + \underbrace{\sigma_{\xi}^2 \; \frac{r^2 (d_1 \vee d_2)^2}{n^2\eps^2}}_{a_2}. 
	\end{equation}
\end{thm} 

By Theorem \ref{thm: minimax_lower_bound}, the lower bound of $\operatorname{risk}(\MM_r)$ consists of two terms where the first term $a_1$ accounts for the statistical error and the second term $a_2$ is the cost of privacy. The proof for $a_1$ can be found in \cite{rohde2011estimation} and the \textit{cost of privacy} is deduced in the following proof. 

\begin{proof}[Proof of Theorem \ref{thm: minimax_lower_bound}]
    
We now start proving Theorem \ref{thm: minimax_lower_bound} by score attack argument. Throughout the proof, we assume that $Z = \{z_1, \cdots, z_n\}$ is an i.i.d. sample drawn from $f_M$ and $Z_i^{\prime}$ is a neighbouring data set of $Z$ obtained by replacing $z_i$ with an independent copy $z_i^{\prime}\sim f_M$. Besides, we mainly focus on the case $M\in \MM_{r, d_1}$ and states the result for the case $M \in \MM_{r, d_2}$ in Remark \ref{rmk_M_r_d_2}. Let $$\calA^{(1)}_M(z, A(Z)):= \left< \left[ A(Z)-M\right ]_{[d_1] \times [r]}, S_M(z)\right>. $$ We derive the lower bound of 
$\operatorname{risk}(\MM_{r, d_1}):=\inf_{A}\sup_{M\in\MM_{r, d_1}} \EE \left\|A(Z)-M\right\|_F^2, $ in three steps.  For ease of notation, we define 
\begin{align*}
	A_i^{\prime}:= \mathcal{A}_{M}\left(z_i, A(Z_i^{\prime})\right)\quad {\rm and }\quad A_i := \mathcal{A}_{M}\left(z_i, A(Z)\right). 
\end{align*}
\noindent\emph{Step 1: bounding the summation.} The following Lemma \ref{lem: soundness} bounds $\EE \left| A_i^{\prime} \right|$; Lemma \ref{lem: upper_bound_summation_A_i} develops the upper bound of $\sum_{i \in[n]} \mathbb{E}\, A_i$ based on $\EE \left| A_i^{\prime} \right|$ discussed in Lemma \ref{lem: soundness} and a tunning parameter $T$.  The proof of Lemma \ref{lem: soundness} and \ref{lem: upper_bound_summation_A_i} can be found in Appendix \ref{proof_soundness} and \ref{proof_lemmas_lower_bounds}. 
\begin{lem}\label{lem: soundness}
	For $i\in[n]$, we have $\EE A_i^{\prime} = 0 \quad {\rm and}\quad \EE \left| A_i^{\prime}\right| \leq \sqrt{\EE \left\|A(Z) - M\right\|^2_F}\sqrt{\frac{C_u}{\sigma_{\xi}^2}}. $
\end{lem}
\begin{lem}\label{lem: upper_bound_summation_A_i}
	Let $A$ be an $(\eps, \delta)$-DP algorithm with $0 < \eps < 1$ and $\delta \geq 0$, under model (\ref{equ: trace_regression_model}), by choosing $T = \sqrt{2/\sigma_{\xi}^2} r d_1 \sqrt{\log(\frac{1}{\delta})}$, we have 
	\begin{equation}\label{upper_bound_summation_A_i}
		\begin{aligned}
			\sum_{i \in[n]} \mathbb{E} A_i \leq 2 n \varepsilon \sqrt{\mathbb{E}\|A(Z)-M\|_F^2} \sqrt{C_u / \sigma_{\xi}^2} + 4 \sqrt{2} \delta r d_1 \sqrt{ \log (1 / \delta) / \sigma_{\xi}^2}.
		\end{aligned}
	\end{equation}
\end{lem}

\noindent\emph{Step 2: lower bounding the summation.}
Under some regularity conditions, the following Lemma \ref{lem: completeness} characterize the quantity $\sum_{i\in[n]} A_i $ as a summation of functions of $M$. Lemma \ref{bayesian_lower_bound_of_summation_over_partial_differentiation} lower bounds the summation of functions by assigning an appropriate prior distribution $\pi$ to $M$.  The proof of Lemma \ref{lem: completeness} and \ref{bayesian_lower_bound_of_summation_over_partial_differentiation} can be found in Appendix \ref{proof_lemmas_lower_bounds}.  

\begin{lem}\label{lem: completeness}
	If for every $(i, j) \in [d_1]\times [r], \log f_{M}(Z)$ is continuously differentiable with respect to $M_{ij}$ and $\left|\frac{\partial}{\partial M_{ij}} \log f_{M}(Z)\right|<h_{ij}(Z)$ such that $\mathbb{E}\left|h_{ij}(Z) A(M)_{ij}\right|<\infty$, we have
	$$ \sum_{i \in[n]} \mathbb{E}\, \mathcal{A}^1_{M}\left(z_i, A(Z)\right) = \sum_{(i, j) \in [d_1]\times [r]}  \frac{\partial}{\partial M_{ij}} \mathbb{E} A(Z)_{ij}. $$
\end{lem}
Lemma \ref{lem: completeness} has its general form stated in Theorem $2.1$,  \cite{cai2023score}. Let $g_{ij}$ be a function defined by $g_{ij}(M):= \left(\EE_{Z| M} A(Z)\right)_{ij}$ for all $(i, j)\in [d_1]\times [r]$, then  $$\sum_{(i, j) \in [d_1]\times [r]}  \frac{\partial}{\partial M_{ij}} \mathbb{E} A(Z)_{ij} = \EE_{\pi} \left( \sum_{(i, j) \in [d_1]\times [r]}  \frac{\partial}{\partial M_{ij}} g_{ij}\right). $$ Lemma \ref{bayesian_lower_bound_of_summation_over_partial_differentiation} lower bounds this quantity by assigning the prior distribution $\pi$ to $M$ such that $M_{ij}\sim \calN(0, 1)$ for all $(i, j) \in [d_1]\times [r]$ and otherwise,  $M_{ij}= 0$. 

\begin{lem}\label{bayesian_lower_bound_of_summation_over_partial_differentiation}
	Let $M\in\MM_{r, d_1}$ be distributed according to a density $\pi$ whose marginal densities are $\{\pi_{ij}\}$ for $i = 1, \cdots, d_1$ and $j = 1, \cdots, d_2$ such that $\pi_{ij}\sim \calN(0, 1)$ for all $(i, j) \in [d_1]\times [r]$, and otherwise, $\pi_{ij}$ be the density function such that $\PP(M_{ij} = 0) = 1$. Then, 
	$$ \EE_{\pi} \left(\sum_{(i, j) \in [d_1]\times [r]}  \frac{\partial}{\partial M_{ij}} g_{ij} \right)\geq  \sum_{(i, j) \in [d_1]\times [r]}  \EE_{\pi_{ij}} M_{ij}^2 - \sqrt{C}\sqrt{\EE_{\pi_{ij}} M_{ij}^2} = rd_1 - \sqrt{C r d_1} \gtrsim rd_1. $$
\end{lem}
Combining Lemma \ref{lem: completeness} and \ref{bayesian_lower_bound_of_summation_over_partial_differentiation}, we obtain
\begin{equation} \label{equ: lower_bound_summation_A_i}
	\sum_{i \in[n]} \EE A_i = \sum_{i \in[n]} \mathbb{E}\, \mathcal{A}_{M}\left(z_i, A(Z)\right) \gtrsim r d_1.
\end{equation}

\noindent\emph{Step 3: combining the upper and lower bounds. }
Combining the lower bound  $(\ref{equ: lower_bound_summation_A_i})$ of $\sum_{i\in[n]} \EE A_i$ and the upper bound (\ref{upper_bound_summation_A_i}) of $\sum_{i\in[n]} \EE A_i$, we have 
$$ 2 n \varepsilon \sqrt{\mathbb{E}_{\pi} \EE_{Z\mid M}\|A(Z)-M\|_F^2} \sqrt{C_u / \sigma_{\xi}^2} \gtrsim rd_1 - 4 \sqrt{2} \delta r d_1 \sqrt{ \log (1 / \delta) / \sigma_{\xi}^2}. $$ 
Under the assumption that $\delta< n^{-(1+\gamma)}$ for some $\gamma>0$, we have $  rd_1 - 4 \sqrt{2} \delta r d_1 \sqrt{ \log (1 / \delta) / \sigma_{\xi}^2}\gtrsim r d_1$, and therefore $\EE_{\pi} \EE_{Z\mid M}\left\| A(Z) - M \right\|_F^2 \gtrsim \sigma_{\xi}^2 \cdot \frac{ r^2 d_1^2}{n^2 \eps^2}. $ Since the sup-risk is greater than the Bayesian risk, 
$$\sup_{M\in \MM_{r, d_1}}\EE_{Z\mid M}\left\| A(Z)- M \right\|_F^2 \gtrsim \sigma_{\xi}^2 \cdot \frac{ r^2 d_1^2}{n^2 \eps^2}. $$ 
Furthermore, due to $\MM_{r, d_1}\subset \bigcup_{k=1}^{r} \MM_k$, we have 
$\sup_{M\in \bigcup_{k=1}^{r} \MM_k}\EE_{Z\mid M}\left\| A(Z)- M \right\|_F^2 \gtrsim \sigma_{\xi}^2 \cdot \frac{ r^2 d_1^2}{n^2 \eps^2} $ and 
$$\inf_{A} \sup_{M\in \bigcup_{k=1}^{r} \MM_k}\EE_{Z\mid M}\left\| A(Z)- M \right\|_F^2 \gtrsim \sigma_{\xi}^2 \cdot \frac{ r^2 d_1^2}{n^2 \eps^2}, $$ 
where $A$ is an $(\eps, \delta)$-DP algorithm that satisfies $\EE_{Z\mid M} \left\| 
A(Z) - M \right\|_F^2 \lesssim 1$. This conclusion extends to all differentially private $A$ if we assume that $r d_1 \lesssim n\eps$ such that $\frac{r^2 d_1^2}{n^2\eps^2}\lesssim 1$. 

\begin{remark}\label{rmk_M_r_d_2}
	Lemma \ref{lem: upper_bound_summation_A_i} and \ref{bayesian_lower_bound_of_summation_over_partial_differentiation} are also applicable to the case where the parameter space is $M_{r, d_2}$. For $M\in \MM_{r, d_2}$, Lemma \ref{lem: upper_bound_summation_A_i} implies that $$ \sum_{i \in[n]} \mathbb{E} A_i \leq 2 n \varepsilon \sqrt{\mathbb{E}\|A(Z)-M\|_F^2} \sqrt{C_u / \sigma_{\xi}^2} + 4 \sqrt{2} \delta r d_2 \sqrt{ \log (1 / \delta) / \sigma_{\xi}^2}; $$
	and Lemma \ref{bayesian_lower_bound_of_summation_over_partial_differentiation} results in $ \EE \left( \sum_{(i, j) \in [d_1] \times [r]}  \frac{\partial}{\partial M_{ij}} g_{ij} \right) \gtrsim rd_2. $ Therefore, as $\delta< n^{-(1+\gamma)}$ for some $\gamma>0$, the minimax lower bound $$ \inf_{A} \sup_{M\in \bigcup_{k=1}^{r} \MM_k}\EE_{Z\mid M}\left\| A(Z)- M \right\|_F^2 \gtrsim \sigma_{\xi}^2 \cdot \frac{ r^2 d_2^2}{n^2 \eps^2}, $$
	where $A$ is an $(\eps, \delta)$-DP algorithm that satisfies $\EE_{Z\mid M} \left\| 
	A(Z) - M \right\|_F^2 \lesssim 1$. Similarly, this conclusion extends to all differentially private $A$ if we assume that $r d_2 \lesssim n\eps$ such that $\frac{r^2 d_2^2}{n^2\eps^2}\lesssim 1$. 
\end{remark} 
\end{proof}

\section{Proofs of Technical Lemmas} \label{appx: technics}
\subsection{Proof of Lemma \ref{lem: RIP_condition}} \label{appx: lem_rip}
Lemma \ref{lem: RIP_condition} is a consequence of Proposition $10.4$, \cite{vershynin2015estimation} and Lemma $1$, \cite{chen2019non} by setting $c_l^2 = C_l$ and $c_u^2 = C_u$ and $\tau^2 \asymp \frac{\sqrt{C_u}}{\sqrt{C_l}}$. See the definiton of $c_l^2$ and $c_u^2$ in  Lemma $1$, \cite{chen2019non}. 
\subsection{Proof of Lemma \ref{lem:spectral-formula}} \label{proof_lem_spectral-formula}
The proof of Lemma \ref{lem:spectral-formula} involves applying the symmetric dilation trick to Theorem $1$, \cite{xia2021normal}. 
\begin{lem}[Theorem 1, \cite{xia2021normal}] \label{theorem_xia}
	Let $B\in d\times d$ be a rank-$r$ symmetric matrix with eigen-decomposition of the form $B = \Theta \Lambda \Theta^{\top}$ where $\Theta\in \OO_{d, r}$ and the diagonal matrix $\Lambda = \{\lambda_1, \cdots, \lambda_r\}$ has the eigenvalues of $B$ arranging in the non-increasing order. Let $\what B = B + \Delta_B$ be another $d\times d$ symmetric matrix and leading $r$ eigen vector of $\what B$ is given by $\what \Theta \in \OO_{r, d}$. Then, 
	$$\what \Theta \what \Theta^{\top} - \Theta \Theta^{\top} = \sum_{k \geq 1} \mathcal{S}_{B, k}(\Delta_B), $$
	where the $k$-th order term $\mathcal{S}_{M_*, k}(\Delta)$ is a summation of $\binom{2k}{k}$ terms defined by
	\begin{equation*}
		\mathcal{S}_{B, k}(\Delta_B)=\sum_{\mathbf{s}: s_1+\ldots+s_{k+1}=k}(-1)^{1+\tau(\mathbf{s})} \cdot Q^{-s_1} \Delta_B Q^{-s_2} \ldots \Delta_B Q^{-s_{k+1}}, 
	\end{equation*}
	where $\mathbf{s}=\left(s_1, \ldots, s_{k+1}\right)$ contains non-negative indices and $\tau(\mathbf{s})=\sum_{j=1}^{k+1} \mathbb{I}\left(s_j>0\right).$ 
\end{lem}

Lemma \ref{theorem_xia} provides an explicit representation formula for the spectral projector $\what \Theta \what \Theta^{\top}$ given that $B$ is symmetric and of rank-$r$. Since we are interested in the asymmetric rank-$r$ matrix $M= U\Sigma V^{\top} \in \RR^{d_1\times d_2}\in \RR^{d_1\times d_2}$, we apply the symmetric dilation trick to $M$ and obtain the rank-$2r$ symmetric matrix $M_*$  has eigendecomposition of the form 
$$
M^* = U_{M^*} \Sigma_{M^*} U_{M^*}^{\top} = \frac{1}{\sqrt{2}}\left(\begin{array}{cc}
	U & U \\
	V & -V
\end{array}\right)\left(\begin{array}{cc}
	\Sigma & 0 \\
	0 & -\Sigma
\end{array}\right) \frac{1}{\sqrt{2}}\left(\begin{array}{cc}
	U & U \\
	V & -V
\end{array}\right)^{\top}. 
$$
The proof is finished by applying Lemma \ref{theorem_xia} with $B=M_*$, $\what B = M_* + \Delta_*$, $d = d_1+d_2$, the rank be $2r$ and 
\begin{align*}
	\Theta = \frac{1}{\sqrt{2}}\left(\begin{array}{cc}
		U & U \\
		V & -V
	\end{array}\right) \quad {\rm and} \quad 
	\what \Theta = \frac{1}{\sqrt{2}}\left(\begin{array}{cc}
		\what U & \what U \\
		\what V & -\what V
	\end{array}\right) . 
\end{align*}

\subsection{Proof of Lemma \ref{lem:dp-fano} and \ref{lem: matrix_permutation}}
See the proof of Lemma \ref{lem:dp-fano} in \cite{acharya2021differentially} and \cite{cai2024optimal}. See the proof of \ref{lem: matrix_permutation} in \cite{shen2023computationally}. 

\subsection{Proof of Lemma \ref{lem:con-bounds}} \label{proof_lem_cond_bounds}
\begin{align*}
	\max_{i\in[n]}\big\|\Delta - \Delta^{(i)}\big\| & \leq \frac{1}{n} \left\|\operatorname{mat}(\Lambda_i^{-1}\operatorname{vec}(X_i)) \left( \left< X_i, M \right> + \xi_i\right)\right\| \\
	& \quad + \max_{i\in[n]}\, \frac{1}{n} \left\|   \operatorname{mat}((\Lambda_i^{\prime})^{-1}\operatorname{vec}(X_i^{\prime})) \left( \left< X_i^{\prime}, M \right> + \xi_i\right)\right\|, 
\end{align*}
where $\|\xi_i\|_{\psi_2} = \sigma_{\xi}$ and $	\left\langle X_i, M\right\rangle \sim N\left(0, \operatorname{vec}\left(M\right)^{\top} \Lambda_i \operatorname{vec}\left(M\right)\right), \quad \Lambda_i^{-1} \operatorname{vec}\left( X_i\right) \sim N\left(0, \Lambda_i^{-1}\right). $
\begin{align*}
	& \left\|\|\xi_i + \left\langle X_i, M\right\rangle \operatorname{mat}\left(\Lambda_i^{-1} \operatorname{vec}\left( X_i\right)\right)-M\|\right\|_{\Psi_1} \\ & \leq\left\|\xi_i +\left\langle X_i, M\right\rangle\right\|_{\Psi_2} \left\|\| \operatorname{mat}\left(\Lambda_i^{-1} \operatorname{vec}\left( X_i\right)\right)\|\right\|_{\Psi_2}  \leq c_0 \sqrt{C_l^{-1}}\left(\sigma_{\xi} + \sqrt{C_u} \sqrt{r} \sigma_1\right), 
\end{align*}
for some absolute constant $c_0>0$. Therefore, for some absolute constant $C_3>0$, with probability at least $1- n^{-10}$, 
$$\left \|\operatorname{mat}(\Lambda_i^{-1}\operatorname{vec}(X_i)) \left( \left< X_i, M \right> + \xi_i\right)\right \|  \leq C_3  \sqrt{ C_l^{-1}}\left(\sigma_{\xi} + \sqrt{C_u} \sqrt{r} \sigma_1\right) \log n. $$
Taking maximum over $n$, with probability at least $1-n^{-9}$ 
\begin{align*}
	\max_{i\in[n]}\big\|\Delta - \Delta^{(i)}\big\| \leq C_3\cdot  n^{-1}\sqrt{ C_l^{-1}}\left( \sqrt{C_u r}\sigma_1 + \sigma_{\xi}\right) \log n. 
\end{align*}
In (\ref{Delta_bounded_by_bernstein}), we have already shown that that for some absolute constant $C_1 >0$, 
\begin{equation}
	\|\Delta\| = \|\what L -M\| \leq C_1 \sqrt{C_l^{-1}} \left( \sigma_{\xi} + \sqrt{C_u} \sqrt{r} \sigma_1 \right) \sqrt{\frac{(d_1 \vee d_2)\log (d_1 + d_2)}{n}} , 
\end{equation} 
with probability at least $1-(d_1+d_2)^{-10}$. Note that for all $i \in [n]$, $ \|\Delta^{(i)}\| = \|\Delta - (\Delta - \Delta^{(i)})\|\leq \|\Delta\| + \|\Delta - \Delta^{(i)}\|, $ and thus
\begin{equation*}
	\|\Delta\|+\max_{i\in[n]}\|\Delta^{(i)}\| \leq 2 \|\Delta\| + \max_{i\in[n]} \|\Delta - \Delta^{(i)} \|. 
\end{equation*}
As long as the sample size $    n \geq \frac{\log^2 n}{(d_1 \vee d_2)\log (d_1+d_2)}, $ there exists an absolute constant $C_0>0$ such that
\begin{equation*}
	\|\Delta\|+\max_{i\in[n]}\|\Delta^{(i)}\| \leq C_0 \sqrt{C_l^{-1}} \left( \sigma_{\xi} + \sqrt{C_u} \sqrt{r} \sigma_1 \right) \sqrt{\frac{(d_1 \vee d_2)\log (d_1 + d_2)}{n}}, 
\end{equation*}
with probability at least $1-(d_1+d_2)^{-10} - n^{-9}$. 

\subsection{Proof of Lemma \ref{local_packing_set}}\label{proof_of_local_packing_set}
By \cite[Proposition 8]{pajor1998metric} and \cite[Lemma 5]{koltchinskii2015optimal},  for any $q\in[1, \infty]$,  there exists an absolute constant $c'>0$ and a subset $\mcS_q^{(d-r)}\subset \OO_{d-r, r}$ such that for any $V_i\neq V_j\in \mcS_q^{(d-r)}$, $\|V_iV_i^{\top}-V_jV_j^{\top}\|_{q}\geq c'r^{1/q}$, and the cardinality of $\mcS_q^{(d-r)}$ is at least $2^{r(d-r)}$.  Here, $\|\cdot\|_q$ denotes the Schatten-$q$ norm of a matrix.  In particular,  spectral norm is Schatten-$\infty$ norm,  Frobenius norm is Schatten-$2$ norm,  and nuclear norm is Schatten-$1$ norm.  Let $\varepsilon_0>0$ be a small number to be decided later.  Now,  for each $V\in\mcS_q^{(d-r)}$,  we define
$$
U=\left(
\begin{array}{c}
	\sqrt{1-\varepsilon_0^2}I_{r}\\
	\sqrt{\varepsilon_0^2}V
\end{array}
\right)
$$
such that $U\in\RR^{d\times r}$ and $U^{\top}U=I_r$.  This means that,  for any $V\in\mcS_q^{(d-r)}$,  we can construct a $U\in\OO_{d, r}$.  This defines a subset $\mcS_q^{(d)}\subset \OO_{d,r}$ with ${\rm Card}\big(\mcS_q^{(d)}\big)\geq 2^{r(d-r)}$ such that for any $U_i\neq U_j\in \mcS_q^{(d)}$,
\begin{align*}
	\|U_iU_i^{\top}-U_jU_j^{\top}\|_q\geq \sqrt{\varepsilon_0^2(1-\varepsilon_0^2)} \|V_i-V_j\|_q\gtrsim \sqrt{\varepsilon_0^2(1-\varepsilon_0^2)}\|V_iV_i^{\top}-V_jV_j^{\top}\|_q\gtrsim \sqrt{\varepsilon_0^2(1-\varepsilon_0^2)} r^{1/q}
\end{align*}
and,  meanwhile, 
$$
\|U_iU_i^{\top}-U_jU_j^{\top}\|_{\rm F}\lesssim \|U_i-U_j\|_{\rm F}\leq \varepsilon_0 \|V_i-V_j\|_{\rm F}\leq \sqrt{2r}\varepsilon_0.
$$

\subsection{Proof of Lemma \ref{lem: sub-Gaussian}, \ref{lem: norm_of_X} and \ref{lem: element_based on_event_RIP}}
See the proof of Lemma \ref{lem: sub-Gaussian} in \cite{vershynin2018high}, Lemma \ref{lem: norm_of_X} in \cite{shen2023computationally} and Lemma \ref{lem: element_based on_event_RIP} in \cite{wei2016guarantees} and  \cite{luo2022tensor}. 

\subsection{Proof of Lemma \ref{lem: soundness}} \label{proof_soundness}
Since $Z_i^{\prime}$ is independent of $z_i$ and $ \EE\, \calS_M(z_i) =  \EE\, \nabla_M f_M(y_i | X_i) = 0$, 
\begin{align*}
	\EE\, \calA^1_M(z_i, A(Z_i^{\prime}))  = \EE \left< \left[A(Z_i^{\prime}) - M\right]_{[d_1] \times [r]}, \calS_M(z_i)\right> = \left< \EE  \left[A(Z_i^{\prime}) - M\right]_{[d_1] \times [r]}, \EE\, \calS_M(z_i)\right> = 0, 
\end{align*}
As for $\EE A_i = \EE\, \calA^1_M(z_i, A(Z))$, we apply Jensen's inequality and have
\begin{equation}\label{upper_E_A_i_prime}
	\begin{aligned}
		& \EE \left| \calA^1_M(z_i, A(Z_i^{\prime}))\right| \leq \sqrt{  \EE \left| \calA^1_M(z_i, A(Z_i^{\prime}))\right|^2 } \\
		& \leq \sqrt{ \EE \operatorname{vec}\left( \left[A(Z_i^{\prime}) - M\right]_{[d_1] \times [r]} \right)^{\top} \operatorname{vec}(\calS_M(z_i))\operatorname{vec}(\calS_M(z_i))^{\top} \operatorname{vec}\left( \left[A(Z_i^{\prime}) - M\right]_{[d_1] \times [r]} \right) } \\
		& \leq \sqrt{ \left\| \EE \operatorname{vec}(\calS_M(z_i))\operatorname{vec}(\calS_M(z_i))^{\top} \right\| } \cdot \sqrt{\EE \left\| \left[A(Z_i^{\prime}) - M\right]_{[d_1] \times [r]}  \right\|_F^2}, 
	\end{aligned}
\end{equation}
where the second line is due to $\left<B, C\right> = \operatorname{vec}(B)^{\top}\operatorname{vec}(C)$ and the last inequality is because $Z_i^{\prime}$ is independent of $z_i$. By the definition of $\calS_M(z_i) = \frac{1}{\sigma_{\xi}^2}(y_i-\left< X_i, M\right>)X_i$ and the independence between $\xi_i$ and $X_i$,  
\begin{align*}
	\left\| \EE \operatorname{vec}(\calS_M(z_i))\operatorname{vec}(\calS_M(z_i))^{\top} \right\| = \left\| \EE \left(\frac{y_i - \left<X_i, M\right>}{\sigma_{\xi}^2} \right)^2 \EE \operatorname{vec}\left(X_i\right)\operatorname{vec}\left(X_i\right)^{\top} \right\|  = \frac{1}{\sigma_{\xi}^2} \|\Lambda_i\| \leq \frac{C_u}{\sigma_{\xi}^2}. 
\end{align*}
Plugging the upper bound of $\left\| \EE \operatorname{vec}(\calS_M(z_i))\operatorname{vec}(\calS_M(z_i))^{\top} \right\| $ into (\ref{upper_E_A_i_prime}), 
\begin{align*}
	\EE \left| \calA^1_M(z_i, A(Z_i^{\prime}))\right| \leq \sqrt{\EE \left\|A(Z) - M\right\|^2_F}\sqrt{\frac{C_u}{\sigma_{\xi}^2}}, 
\end{align*}

\subsection{Proof of Lemma \ref{lem: upper_bound_summation_A_i}, Proof of Lemma \ref{lem: completeness}, Lemma \ref{bayesian_lower_bound_of_summation_over_partial_differentiation}} \label{proof_lemmas_lower_bounds}
Lemma \ref{lem: upper_bound_summation_A_i} is a trivial consequence by setting $T = \sqrt{2/\sigma_{\xi}^2} r d_1 \sqrt{\log(\frac{1}{\delta})}$ to Proposition $2.1$, \cite{cai2023score}. Lemma \ref{lem: completeness} is a trivial consequence of Theorem 2.1, \cite{cai2023score} along with the definition of $\calA_M^1\left(z_i, A(Z)\right)$. Lemma \ref{bayesian_lower_bound_of_summation_over_partial_differentiation} is a trivial consequence of Proposition $2.2$, \cite{cai2023score} by taking $M_{ij}\sim \calN(0, 1)$ for $(i, j)\in [d_1]\times [r]$ and $M_{ij} = 0$ otherwise. 
\subsection{Proof of Lemma \ref{theorem_xia}}
See the proof of Lemma \ref{theorem_xia} in  \cite{xia2021normal}.

\section{Weak Differential privacy} \label{appx: weak_dp} 
This section proposes a weaker definition than differential privacy such that the sensitivities are \textit{free of}  $\{X_i\}_{i\in[n]}$. 

\begin{defn}[\textit{weak} $(\varepsilon, \delta)$-differential privacy] \label{def: weak_DP}
	Let $Z$ be a given data set and $Z'$ be a \textit{weak} neighbouring data set of $Z$, i.e., $Z$ and $Z'$ differs by at most one pair of observations $z\in Z$ and $z^{\prime}\in Z^{\prime}$ sharing the same measurement $X$. The algorithm $A$ that maps $Z$ into $\RR^{d_1 \times d_2}$ is \textit{weak} $(\varepsilon, \delta)$-differentially private over the dataset $Z$ if 
	\begin{equation} \label{eps_delta_ineq}
		\PP\big(A(Z)\in \calQ\big)\leq e^{\eps}\PP\big(A(Z')\in\calQ\big)+\delta,
	\end{equation}
	for all \textit{weak} neighbouring data set $Z, Z'$ and all subset $\calQ\subset \RR^{d_1\times d_2}$. 
\end{defn} 
Compared to the standard $(\varepsilon, \delta)$-DP, \textit{weak} $(\varepsilon, \delta)$-differential privacy is a less powerful constraint. Definition \ref{def: weak_DP} only requires the algorithm $A$ to preserve the property (\ref{eps_delta_ineq}) over weak neighbouring datasets, i.e., datasets that differs by at most one pair of observations sharing the same measurement $X$. As we consider a pair of observations $z = (X, y)$ and $z^{\prime}= (X, y^{\prime})$  under the model (\ref{equ: trace_regression_model}),  where $	y = \left<X, M\right>+\xi \quad {\rm and} \quad y^{\prime} =  \left<X, M\right>+\xi^{\prime}, $ the difference $y - y^{\prime} = \xi - \xi^{\prime}$ is free of the measurement $X$. 

Next, we list the Theorem \ref{thm: weak_dp_init}, Corollary \ref{cor: weak_good_init} and Theorem \ref{thm: weak_dp_convergence_rate} as the analogues of Theorem \ref{thm: dp_init}, Corollary \ref{cor: good_init} and Theorem \ref{thm: convergence_rate}. All proofs for this section are deferred to the end of this section. 

\begin{thm} [Weak DP and utility guarantees of the initialization $\wt M_0$] \label{thm: weak_dp_init}
	Consider i.i.d. observations $Z = \{z_1, \cdots, z_n\}$ drawn from the trace regression model stated in $(\ref{equ: trace_regression_model})$ where $z_i := (X_i, y_i)$ for $i=1, \cdots, n$. Let the true low-rank regression coefficients matrix being $M\in\MM_r$. Suppose that $\{X_i\}_{i\in[n]}$ satisfy the Assumption \ref{assumption_for_X}. Under the mild condition $n \geq \frac{\sigma_{\xi}}{\sigma_{\xi}+\sqrt{C_u r}\sigma_1}$, there exists absolute constants $C_1, C_2, C_3>0$ such that as the sample size $ n\geq n_0 $, the sensitivity for leading $r$ left and right singular vectors takes the value  
	\begin{align*}
		\Delta_{weak}^{(1)} & := \max_{i\in[n]} \left( \|  \what U \what U^{\top} - \what U^{(i)} \what U^{(i)\top} \|_F \vee \| \what V \what V^{\top} - \what V^{(i)} \what V^{(i)\top} \|_F \right) = C_2 \sqrt{C_l^{-1}} \frac{\sigma_{\xi}  }{\sigma_r} \frac{\sqrt{r}}{n}\log n;  
	\end{align*}
	the sensitivity for the $r$ singular values takes the value 
	\begin{align*}
		\Delta_{weak}^{(2)} & : = \max_{i\in [n]} \left\| \wt U^{\top} \left( \what L - \what L^{(i)} \right) \wt V \right\|_F = C_2 \sqrt{C_l^{-1}} \sigma_{\xi} \frac{\sqrt{r}}{n} \log n, 
	\end{align*} 
	and Algorithm \ref{alg:DP-init} is weak $(\eps, \delta)$-differentially private. Moreover, 
	\begin{align*}
		\| \wt M_0^{weak} - M \|  \bigvee  \left(\| \wt M_0^{weak} - M \|_F/ \sqrt{2r} \right) & \leq e_1 + \underbrace{C_3  \sqrt{C_l^{-1}}  \sigma_{\xi} \left( \frac{\sigma_1}{\sigma_r}  \frac{\sqrt{r (d_1 \vee d_2)}}{n\varepsilon} + \frac{r}{n\eps} \right)  \log n \log^{\frac{1}{2}}(\frac{3.75}{\delta})}_{e^{weak}_2} , 
	\end{align*}
	with probability at least $1- (d_1+d_2)^{-10} - n^{-9} - \exp(-d_1) - \exp(-d_2) - 10^{-20r}$. 
\end{thm}

Theorem \ref{thm: weak_dp_init} requires the same sample size condition $n\geq n_0$ as Theorem \ref{thm: dp_init}, however, the sensitivities $\Delta_{weak}^{(1)}$ and  $\Delta_{weak}^{(2)}$ derived under weak DP, differs with their DP counterpart $\Delta^{(1)}$ and $\Delta^{(2)}$ by the factor $\sqrt{C_u r}\sigma_1$. This leads to a smaller cost of privacy $e^{weak}_2$ than the cost of privacy  $e_2$ we obtained under stronger standard DP-constraints, as presented in Theorem \ref{thm: dp_init}.  

\begin{cor}\label{cor: weak_good_init}
	Under the conditions stated in Theorem \ref{thm: weak_dp_init}, as the sample size is sufficiently large such that for some absolute constant $c_2>0$, 
	\begin{align*}
		n\geq C_1 \max \Bigg\{ n_1,  \underbrace{ \sqrt{C_l^{-1}}  \left(\frac{ \sigma_{\xi} }{\sigma_r}\right) \left( \kappa r\sqrt{d_1\vee d_2} + r^{\frac{3}{2}}\right) \log n \frac{\log^{\frac{1}{2}}(\frac{3.75}{\delta})}{\varepsilon}}_{n_2^{weak}}\Bigg \}, 
	\end{align*} 
	we have for some small constant $0<c_0<1$, $\| \wt M_0 - M \|_F \leq \sqrt{2r} \| \wt M_0 - M \| \leq c_0 \sigma_r. $
\end{cor} 

Compared with Corollary \ref{cor: good_init}, Corollary \ref{cor: weak_good_init} requires smaller sample size as $n_2^{weak}\leq n_2$. 

\begin{thm} \label{thm: weak_dp_convergence_rate}
	Consider i.i.d. observations $Z = \{z_1, \cdots, z_n\}$ drawn from the trace regression model stated in $(\ref{equ: trace_regression_model})$ where the true low-rank regression coefficients matrix being $M\in\MM_r$. Here, $z_i := (X_i, y_i)$ for $i=1, \cdots, n$ and we assume that $\{X_i\}_{i\in[n]}$ satisfy the Assumption \ref{assumption_for_X} and $(d_1+d_2)>\log n$. Suppose the weak $(\varepsilon, \delta)$-DP initialization satisfies \ref{equ: good_init}, then Algorithm \ref{alg:DP-RGrad} is weak $(2\eps, 2\delta)$-differentially private with the sensitivities 
	\begin{align*}
		\Delta_l & = C_3 \frac{\eta}{n} \sigma_{\xi} \sqrt{C_u r (d_1+d_2)\log n}, 
	\end{align*}
	for some absolute constant $C_3>0$. Moreover, as the sample size 
	\begin{align*}
		n \geq c_4 \max \Bigg\{ n_3, n_4, \underbrace{ \eta \sqrt{C_u } \kappa_{\xi} r (d_1+d_2) \log^{3/2} (n)  \frac{ \log^{1/2}\left( \frac{1.25 \log(n)}{\delta}\right) }{\eps} }_{n_5^{weak}}\Bigg\}, 
	\end{align*} 
	for some small constant $0<c_4<1$, number of iteration $l^* = O(\log n)$, and the step size $0<\eta<1$, we have the output of Algorithm \ref{alg:DP-RGrad} satisfies  
	\begin{align*}
		\left\|\wt M_{l^*} - M\right\|_F & \leq  u_1 + \underbrace{C_4 \sqrt{C_u } \sigma_{\xi}  \frac{ r (d_1+d_2) }{n\eps} \log^{3/2} n \log^{1/2}\left( \frac{1.25 \log (n)}{\delta} \right)}_{u_2^{weak}} . 
	\end{align*} 
	with probability at least 
	\begin{align*}
		1  &  -\wt c_2 \exp \left(-\wt c_3 r(d_1+d_2)\right) -(d_1+d_2)^{-10}-n^{-9} - e^{-d_1} - e^{-d_2} - 10^{-20r} \\ 
		& - \left( (d_1+d_2)^{-10} + \exp(-(d_1+d_2)) + n^{-9} + \exp\left( - 10 C_u(d_1+d_2) \right)n^{-9} \right)\log n. 
	\end{align*} 
\end{thm}

Theorem \ref{thm: weak_dp_convergence_rate} shows that as the sample size $n\gtrsim \wt O \left( \left(\kappa_{\xi}^2 \vee \kappa_{\xi} \right) r(d_1\vee d_2)  \right), $ the estimator $\wt M_{l^*}$ given by Algorithm \ref{alg:DP-RGrad} attains the optimal convergence rate $	\wt O_p\left( \sigma_{\xi} \sqrt{\frac{ r (d_1+d_2) }{n}}+ \sigma_{\xi} \frac{ r (d_1+d_2) }{n\eps} \right), $ in the sense of weak differential privacy. 

The proofs of Theorem \ref{thm: weak_dp_init}, \ref{thm: weak_dp_convergence_rate} and Corollary \ref{cor: weak_good_init} will be a trivial concequence of replacing the first part of Lemma \ref{lem:con-bounds} by the following Lemma \ref{lem: weak_dp_con-bound}

\begin{lem}\label{lem: weak_dp_con-bound}
	Under model (\ref{equ: trace_regression_model}), Assumption \ref{assumption_for_X}, and the condition $n \geq \frac{\sigma_{\xi}}{\sigma_{\xi}+\sqrt{C_u r}\sigma_1}$, there exists some absolute constant $C_0, C_1>0$ such that the event 
	\begin{align*}
		\calE_*:= & \Bigg\{\max_{i\in[n]}\big\|\Delta - \Delta^{(i)}\big\| \leq C_0\cdot  n^{-1}\sqrt{ C_l^{-1}} \sigma_{\xi} \log n \Bigg\} \\
		& \bigcap \Bigg\{  \|\Delta\|+\max_{i\in[n]}\|\Delta^{(i)}\| \leq C_0 \sqrt{C_l^{-1}} \left( \sigma_{\xi} + \sqrt{C_u} \sqrt{r} \sigma_1 \right) \sqrt{\frac{(d_1 \vee d_2)\log (d_1 + d_2)}{n}} \Bigg\}, 
	\end{align*}
	holds with probability at least $1-(d_1+d_2)^{-10} - n^{-9}$. 
\end{lem}

\begin{proof}[Proof of Lemma \ref{lem: weak_dp_con-bound}] 
	We only need to focus on $	\max_{i\in[n]}\big\|\Delta - \Delta^{(i)}\big\|$ since the rest of the proof is the same as Lemma \ref{lem:con-bounds}. 
	\begin{align*}
		\max_{i\in[n]}\big\|\Delta - \Delta^{(i)}\big\|  \leq \max_{i\in[n]} \frac{1}{n} \left\|\operatorname{mat}(\Lambda_i^{-1}\operatorname{vec}(X_i)) \left(\xi_i - \xi_i^{\prime} \right)\right\|, 
	\end{align*}
	where $\|\xi_i\|_{\psi_2} = \|\xi_i^{\prime}\|_{\psi_2} = \sigma_{\xi}$ and $\Lambda_i^{-1} \operatorname{vec}\left( X_i\right) \sim N\left(0, \Lambda_i^{-1}\right)$ for all $i = 1, \cdots, n$. 
	Therefore, for some absolute constant $c_0>0$, 
	\begin{align*}
		& \left\| \left\|\operatorname{mat}(\Lambda_i^{-1}\operatorname{vec}(X_i)) \left(\xi_i - \xi_i^{\prime} \right)\right\|  \right\|_{\Psi_1} \leq c_0 \sqrt{C_l^{-1}}\sigma_{\xi} . 
	\end{align*} 
	We complete the proof by appying tail bound for sub-exponential random variable and taking a maximum over $n$.  
\end{proof}

\bibliographystyle{chicago}
\bibliography{DP-Trace-arxiv}

\end{document}